\documentclass[twoside]{article}

\usepackage[preprint]{aistats2026}
\usepackage{algorithm}
\usepackage{physics}
\usepackage{amsmath}
\usepackage{tikz}
\usepackage{mathdots}
\usepackage{yhmath}
\usepackage{cancel}
\usepackage{color}
\usepackage{siunitx}
\usepackage{array}
\usepackage{multirow}
\usepackage{amssymb}
\usepackage{tabularx}
\usepackage{extarrows}
\usepackage{booktabs}
\usepackage{amsthm}
\usepackage{breqn}
\usepackage{float}
\usepackage{subcaption}
\usepackage{amsfonts}
\usepackage{amssymb}
\usepackage{graphicx}
\allowdisplaybreaks
\usepackage[round]{natbib}
\usepackage{hyperref}
\hypersetup{
    colorlinks=true,
    linkcolor=blue,  % You can choose the color you prefer for links
    citecolor=blue,  % Color for citation links
    urlcolor=blue    % Color for URLs
}
\renewcommand{\leq}{\leqslant}
\renewcommand{\geq}{\geqslant}
\newtheorem{thm}{Theorem}
\newtheorem{assumption}{Assumption}
\newtheorem{lema}{Lemma}
\newtheorem{corollary}{Corollary}

\usepackage{pifont}
\usepackage{algcompatible}

\begin{document}

% If your paper is accepted and the title of your paper is very long,
% the style will print as headings an error message. Use the following
% command to supply a shorter title of your paper so that it can be
% used as headings.
%
\runningtitle{A Linear Matching Bandit Approach to Online Human-Robot Matching}

% If your paper is accepted and the number of authors is large, the
% style will print as headings an error message. Use the following
% command to supply a shorter version of the author names so that
% they can be used as headings (for example, use only the surnames)
%
%\runningauthor{Surname 1, Surname 2, Surname 3, ...., Surname n}

\twocolumn[

\aistatstitle{A Linear Matching Bandit Approach to \\ Online Multi-Human Multi-Robot Matching}

\aistatsauthor{ Yaohui Guo \And X. Jessie Yang \And  Cong Shi }

\aistatsaddress{ University of Michigan \\ \{yaohuig, xijyang, shicong\}@umich.edu  } ]

\begin{abstract}
We address the problem of online multi-human multi-robot matching through the lens of a linear matching bandit framework, where a learner assigns robots with unknown features from a fixed pool to distinct sets of human agents over multiple rounds. To solve this problem, we propose {\tt LinMatch}, an online learning algorithm that updates the confidence intervals of the unknown features and makes the optimistic matching under uncertainty. The contributions and novelty of this work are twofold. First, we recast the optimistic matching problem in each round as a linear program of maximum weighted matching, efficiently solvable by the celebrated Hungarian algorithm. Second, we provide novel bounds for matching with linear feature problems, showing that $\Theta(\sqrt{T})$ is the optimal achievable regret with respect to the total number of rounds $T$. The proposed algorithm and bounds apply to a wide range of matching problems with applications beyond human-robot matching, such as housing allocation, recommendation systems, and more.
\end{abstract}

\section{Introduction}
The adoption of multiple-human multiple-robot teams is rising in areas that demand sophisticated coordination and adaptability, such as military operations~\citep{ramchurn2015study, freedy2008multiagent}, factory automation~\citep{gombolay2015coordination, liu2021coordinating}, and automated agriculture systems~\citep{lippi2023optimal, JiRSS22}. In these domains, efficient matching between robots and human agents is crucial for maximizing the team performance in a variety of tasks, including task allocation~\citep{Wang2023}, operators allocation~\citep{dahiya2022scalable, JiRSS22}, matching and routing~\citep{Fu22}, and robot dispatch~\citep{lippi2023optimal}. Two detailed examples are provided in the Appendix.

% Many systems today consist of multiple humans and multiple robots working together \citep{freedy2008multiagent, ramchurn2015study, liu2021coordinating, lippi2023optimal, JiRSS22, Wang2023, dahiya2022scalable}.  Consider the following motivating scenario:

% \textbf{Scenario:} Teleoperation enhances safety in self-driving cars by enabling remote human operators to handle "edge cases" such as severe weather or unclear lane markings. At any given time, multiple vehicles might need help, and a pool of teleoperators are ready to assist. Each teleoperator offers distinct skills and knowledge, tailored to specific car brands or geographic areas where assistance is needed. Similarly, each vehicle may face unique challenges determined by its specific location, operational context, or technical issues. \textbf{Human-robot online teaming:} In environments where multiple cars and teleoperators are present, effectively managing the dynamic teaming/pairing between these teleoperators and vehicles becomes essential. This involves continuously assessing the unique skills of each teleoperator and the specific needs and capabilities of each vehicle, as well as making optimal teaming/pairing decisions. This strategic matching is crucial for maximizing efficiency and enhancing safety, allowing for quicker, more accurate task completion and a significant reduction in errors, thus improving overall operational efficiency.

\begin{figure}[t]
    \centering
    \includegraphics[width = 0.9\columnwidth]{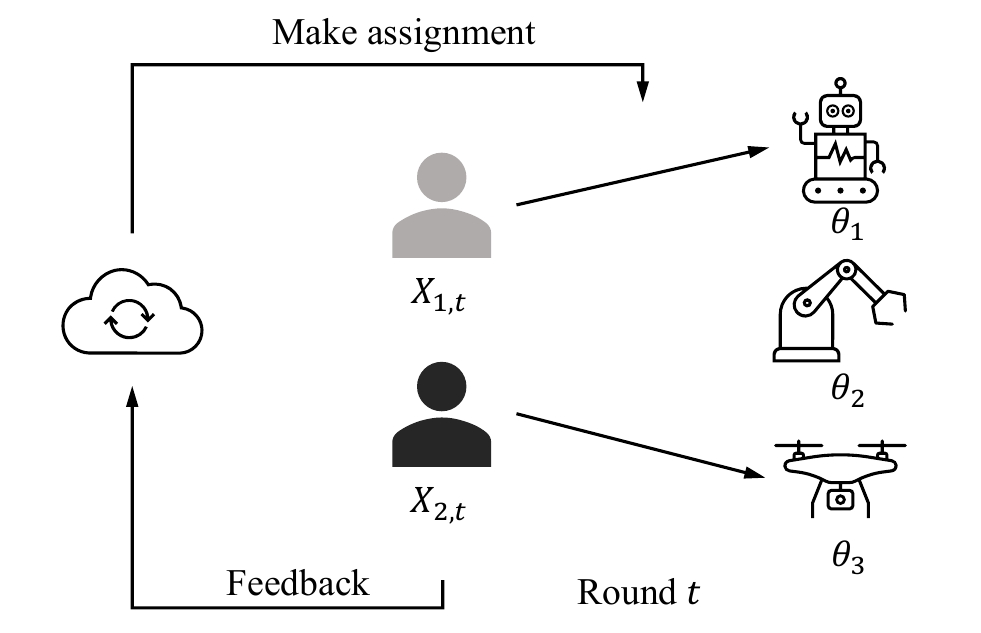}
    \caption{\footnotesize In online multi-human multi-robot matching, a fixed pool of $K$ robots are going to match with human agents for several rounds. In each round, $M$ human agents come, and a learner assigns each human a robot to complete a task. The human agents can vary from round to round. Each human-robot team will produce a reward, which can be observed by the learner. The learner's goal is to maximize the total rewards by making the optimal matching in each round. We assume the expected outcome of each human-robot team is determined by the inner product of their features. The feature of each human agent is known, potentially learned from the human's capability or task context, but the robots' features are unknown due to the complexity of AI systems.
    % To reduce the complexity of the problem, we assume a linear reward structure, i.e., the expected outcome of each human-robot team is determined by the inner product of their features. We assume the feature of each human agent is known, potentially learned from the human's capability or task context; while we assume the robots' features are unknown. The human    
    }
    \vspace{-5mm}
    \label{fig:motivation}
\end{figure}

The challenges in efficient matching arise from the potentially unknown and heterogeneous characteristics of the agents. As a consequence, an ideal matching algorithm will uncover these characteristics and make the optimal matching accordingly. In this work, we investigate the multi-human-multi-robot matching problem within the bandit framework for its effectiveness in online learning problems. As illustrated in Figure \ref{fig:motivation}, we consider the problem that a fixed pool of robots team up with varying human agents across multiple rounds. At each round, our learner allocates robots to human agents to maximize the total reward, thus enhancing the overall system performance. In such a framework, the learner faces two major challenges. Firstly, the varying human features and diverse robot capabilities limit the applicability of insights learned from previous rounds. If we treat each matching as a bandit arm, then the total number of arms grows exponentially with the number of humans and robots, making tracking the performance of each human-robot team intractable. Secondly, the inherent combinatorial complexity complicates the decision process, as rewards are interconnected by the constraints of the matching.

To address the first challenge, we assume a linear reward structure, a common approach in personalized problems like article recommendation~\citep{li2010contextual, li2011unbiased} and healthcare~\citep{tewari2017ads}. In particular, we assume that each robot is characterized by a feature vector $\theta \in \mathbb{R}^d$ such that the expected reward the robot produces when matching with a human agent with feature $X \in \mathbb{R}^d$ is $X^{\top}\theta$. In addition, we assume that the features of robots are \textit{unknown} due to the increasing complexity and lack of transparency in advanced robotic systems; while human features are assumed \textit{known}, which can be gleaned from user profiles or the context of the task at hand. Second, to manage the combinatorial complexity, we employ an upper-confidence-bound (UCB) method to balance exploration and exploitation. In particular, we propose the {\tt LinMatch} algorithm, which in each round estimates the confidence interval of the unknown parameters and makes an optimistic matching in the face of uncertainty. We demonstrate that the optimal matching problem in each round can be efficiently solved as a linear program using the Hungarian algorithm.

\smallskip
\noindent \textbf{Major Contributions.} Although previous literature has studied the matching bandit problem, they either have different mathematical formulations \citep[e.g.,][]{cen2022regret, liu2020competing, li2019online, li2022rate, erginbas2023interactive} or lack theoretical guarantees on regret \citep[e.g.,][]{le2006multi, ul2016efficient, hassan2014multi}. The main contributions are summarized as follows.

\begin{itemize}
\item {\bf A Novel Algorithm for Matching Bandits:} We propose \texttt{LinMatch}, an OFU (Optimism in the Face of Uncertainty) algorithm specifically designed for online linear matching bandits. The algorithm leverages ridge regression to estimate the feature vectors of workers and solves for optimistic matching. A central novelty of \texttt{LinMatch} lies in recasting the combinatorial mixed-integer matching problem as a linear program for maximum-weight matching, which can be efficiently solved using the Hungarian algorithm.
\item {\bf Theoretical Guarantees:} We provide rigorous theoretical results for \texttt{LinMatch}. Specifically, we establish an instance-independent upper bound on regret of $O(d\sqrt{MKT})$, where $d$ is the feature dimension, $M$ is the number of tasks, $K$ is the number of workers, and $T$ is the total number of rounds. Moreover, we prove a minimax lower bound of $\Omega(\sqrt{dKT})$, demonstrating that the $\sqrt{T}$ factor in our upper bound is tight. To the best of our knowledge, this is the first work to show that the instance-independent optimal achievable regret rate for matching problems with linear bandit structure is $\Theta(\sqrt{T})$.
\end{itemize}

Our results offer a deeper understanding of the fundamental limits of online matching bandit problems with linear features and a practical framework for achieving optimal performance in real-world applications such as dynamic task assignment and resource allocation.

\section{Related Work}

Several studies have investigated the matching problem in multi-human multi-robot teams~\cite{Fu22, lippi2023optimal, dahiya2022scalable}. However, our mathematical formulation of matching bandits with linear features is novel, which is most relevant to bandits literature, including linear bandits, ranking bandits, and centralized matching markets, which are summarized in Table~\ref{table:relatedwork}.

\textbf{Linear Bandits.} The linear reward structure in our problem has been extensively studied in linear bandits, which were first introduced by \cite{abe1999associative}, and later studied in both cases of the finite number of arms \citep{auer2002using, chu2011contextual} and the infinite number of arms \citep{rusmevichientong2010linearly, abbasi2011improved}. 
% Our framework gives a new and explicit formulation involving the multi-human multi-robot matching problem, which is a cornerstone problem in human-robot interactions.

\renewcommand{\arraystretch}{1.6}
\begin{table*}[t]
\centering
% Number of workers K <-> number of items L
% number of tasks M <-> number of positions K
\caption{This table summarizes the relevant literature, including linear bandits, ranking bandits, and matching bandits. Note that $M$ is the number of humans (agents), $K$ is the number of robots (arms), $T$ is the number of rounds, and $d$ is the dimension of the feature vectors. The term $\Delta$ is the minimal gap between the expected optimal matching reward and suboptimal matching reward (player-wise in matching markets literature). * indicates stable regret for each player.}
% \vspace{2mm}
\scalebox{0.9}{
\begin{tabularx}{1.1\textwidth}{p{6.5cm} p{8cm} p{4cm}} 
\toprule 
{\bf Related Work} & {\bf Setting} & {\bf Regret} \\ \hline

\cite{abbasi2011improved} & Linear bandits with unit ball action set & $O(d\log (T) \sqrt{T})$ \\ \hline

\cite{auer2002using, chu2011contextual} & Linear bandits with finite arms & $O(\sqrt{Td\log^3(TK)})$ \\ \hline

\cite{lattimore2018toprank} & Ranking bandits with general click model & $O\left(MK\frac{\log(T)}{\Delta}\right)$ \\ \hline

\cite{li2019online} & Ranking bandits with linear general click model &  $O\left(Md\sqrt{T\log(KT)}\right)$ \\ \hline

\cite{liu2020competing} & Centralized two-sided matching market & $O(MK \frac{\log T}{\Delta^2})$ *\\ \hline

\cite{liu2021bandit} & Decentralized two-sided matching market &  $O(M^5K^2 \frac{(\log T)^2}{\Delta^2})$  *\\ \hline

\cite{li2022rate} & Two-sided matching with linear features & $O(2d^2 \frac{\log (\Delta^2 MKT/d)}{\Delta^2})$  *\\ \hline

\cite{erginbas2023interactive} & Centralized matching with linear features, warm start & $O(M\sqrt{KdT})$ \\ \hline

{\bf Ours} & Centralized matching with linear features & $O(d\sqrt{MKT)}$ \\
\bottomrule
\end{tabularx}
}
% \vspace{5mm}
\label{table:relatedwork}
\end{table*}

\textbf{Ranking Bandits.} The online matching problem can be viewed as choosing a maximum weighted matching in a bipartite graph $G:=[ M] \amalg [ K]$. Such a structure is shared in ranking bandits, where the learner chooses an ordered subset of $M$ items from a larger set of $K$ items to present to a user to maximize clicks. The probability of the user clicking the presented items is modeled by a click model to simplify the statistical complexity. Previous research has studied various click models, such as the cascade model \citep{kveton2015cascading} and the dependent-click model \citep{katariya2016dcm}. Follow-up studies decomposed the click model as a product of the examination probability and attractiveness function. Notably, some works assume that the attractiveness function can be represented as a linear function of an item's features, thereby offering a scalable solution for large item sets \citep{li2016contextual, liu2018contextual, li2019online}.

\textbf{Centralized Matching Market.} Another highly related line of research is centralized matching markets, where a central platform tries to match agents from two sides with each other. Matching markets have many practical applications such as college admission \citep{gale1962college}, labor markets \citep{roth1984evolution}, and kidney exchange \citep{roth2005pairwise}. However, only recently, the matching markets model in bandit frameworks was introduced by \cite{liu2020competing}, and later studied in varying settings \citep{jagadeesan2021learning, cen2022regret, erginbas2023interactive}. \citet{li2022rate} adopted a similar setting as ours, where the matching score is the inner product of the features of the two matched agents and the features change round to round. However, while matching markets aim to maximize stable regret by considering the preferences of both sides, our problem focuses on maximizing the total rewards from the matching. Additionally, \citet{erginbas2023interactive} investigated a centralized joint matching and pricing problem, requiring an initial estimation of the mean reward matrix from historical data.

\section{Problem Setting}\label{sec:setting}

We consider a multi-human multi-robot team with $K$ robots and $M$ human agents working for $T$ rounds. At round $t$, $M$ human agents arrive, each characterized by a feature vector $X_{m,t} \in \mathbb{R}^{d}$. The learner observes the entire set of features $\mathcal{X}_{t} :=\{X_{m,t}\}_{m=1}^{M}$ of human agents, as well as the interaction history $\mathcal{H}_{t-1} =(\mathcal{X}_{1},\mathnormal{A}_{1} ,Y_{1} ,\dotsc ,\mathcal{X}_{t-1} ,A_{t-1} ,Y_{t-1})$, then chooses action $A_{t} \in \mathcal{A}_{M,K}$ according to its policy $\pi $, where the action set $\mathcal{A}_{M,K}$ is the collection of all injection functions from $[ M]$ to $[ K]$. This action $A_{t}$ assigns each human $m$ a robot $A_{t}( m)$. To ensure that each human is assigned a robot, we impose the condition that $M\leqslant K$. Upon completing the task, the learner receives a reward $Y_{m,t}$ from team $m$, for all $m\in [ M]$. We denote the entire reward vector for round $t$ as $Y_{t} =( Y_{1,t} ,\dotsc ,Y_{M,t})$. A summary of major notation is presented in the Appendix.

We assume linear realizability for the expected rewards: for all $k\in [ K]$, $m\in [ M]$, and $t\in [ T]$ such that $A_{t}( m) =k$, there exists some unknown vector $\theta _{k}^{*} \in \Theta \subseteq \mathbb{R}^{d}$ of robot $k$ such that $Y_{m,t} =X_{m,t}^{\top } \theta _{k}^{*} +\epsilon _{m,t}$, where $\epsilon _{m,t}$ is a $\sigma $-sub-Gaussian random noise. This sub-Gaussian noise assumption is standard in bandit literature \citep{lattimore2020bandit, li2019online, li2022rate}. Furthermore, we assume the noise terms $\{\epsilon _{m,t}\}_{m=1}^{M}$ are independent of each other given the features $\mathcal{X}_{t}$, matching $A_{t}$, and history $\mathcal{H}_{t-1}$. The optimal matching for round $t$ is given by 
\begin{equation}
\begin{aligned}
A_{t}^{*} = & \arg\max_{A\in \mathcal{A}_{M,K}}\mathbb{E}\left[\sum _{m=1}^{M} Y_{m,t}\right]\\
= & \arg\max_{A\in \mathcal{A}_{M,K}}\sum _{m=1}^{M} X_{m,t}^{\top } \theta _{A( m)}^{*} .
\end{aligned}
\end{equation}
% \begin{equation}
% $A_{t}^{*} = \arg\max_{A\in \mathcal{A}_{M,K}}\mathbb{E}\left[\sum _{m=1}^{M} Y_{m,t}\right]
% = \arg\max_{A\in \mathcal{A}_{M,K}}\sum _{m=1}^{M} X_{m,t}^{\top } \theta _{A( m)}^{*}$.
% \end{equation}
Hence, we define the regret $R_{T}\left( \pi ,\{X_{t}\}_{t=1}^{T}\right)$ for policy $\pi $ given the human agents' features as 
\begin{equation}
% = \sum _{t=1}^{T}\max_{A\in \mathcal{A}_{M,K}}\sum _{m=1}^{M} X_{m,t}^{\top } \theta _{A( m)}^{*} -\mathbb{E}_{\pi }\left[\sum _{t=1}^{T}\sum _{m=1}^{M} Y_{m,t}\right]
\sum _{t=1}^{T}\sum _{m=1}^{M} X_{m,t}^{\top } \theta _{A_{t}^{*}( m)}^{*} -\mathbb{E}_{\pi }\left[\sum _{t=1}^{T}\sum _{m=1}^{M} X_{m,t}^{\top } \theta _{A_{t}( m)}^{*}\right]
\end{equation}
% where $A_{t}^{*} =\arg\max_{A_{t} \in \mathcal{A}_{M,K}}\sum _{m=1}^{M} X_{m,t}^{\top } \theta _{A_{t}( m)}^{*}$ is the optimal matching at round $t$.

We note that our problem formulation accommodates adversarial matchings. Specifically, the feature set $\{\mathcal{X}_{t}\}_{t=1}^{T}$ may be selected and fixed by an oblivious adversary prior to the first round. In addition, we make the following standard assumption on the magnitude of the feature vectors:
\begin{assumption}\label{assum:norm}
$\left\Vert \theta _{k}^{*}\right\Vert \leqslant S$ for all $k\in [ K]$, $\Vert X_{m,t}\Vert \leqslant L$ for all $( m,t) \in [ M] \times [ T]$.
\end{assumption}

\section{LinMatch Algorithm}
We introduce the {\tt LinMatch} algorithm, outlined in Algorithm~\ref{alg:lin}. At each round, {\tt LinMatch} first computes the confidence interval of each parameter $\theta _{k}^{*}$ for all $k\in [ K]$, and then finds the optimistic matching in the face of uncertainty. The core idea of {\tt LinMatch} is that we can formulate the optimistic matching problem as a maximum weighted matching, which can be efficiently solved by the Hungarian algorithm.

\begin{algorithm}[tb]
\caption{{\tt LinMatch}}
\label{alg:lin}
\begin{algorithmic}[1]
\REQUIRE $K,M,d, S, L\in \mathbb{N}$
\STATE $V_{k,0}\leftarrow \lambda I$, $h_{k,0}\leftarrow \mathbf{0}$, for all $k\in [ K]$
\FOR{$t=1,2,...,T$}
    \STATE \hspace{-3mm} \textit{Step 1. Update confidence interval $\hat{\Theta }_{k,t}$'s:}
    \FOR{$k=1,2,...,K$}
        \STATE {\small $\hat{\theta }_{k,t}\leftarrow V_{k,t}^{-1} H_{k,t}$} 
        % \STATE $\rho _{k,t} \leftarrow \sigma \sqrt{2\log\frac{K}{\delta } +d\log(1+\frac{| \Psi _{k,t}| L^{2}}{d\lambda })} +\sqrt{\lambda } S$
        \STATE {\small $\rho_{k,t} \leftarrow \sigma \sqrt{2\log\frac{K}{\delta} + d\log\left(1+\frac{|\Psi_{k,t}| L^{2}}{d\lambda}\right)} + \sqrt{\lambda} S$}
    \ENDFOR
    \STATE \hspace{-3mm} \textit{Step 2. Observe features $X_{m,t} \in \mathbb{R}^{d}$, $\forall m \in [M]$}
    \STATE \hspace{-3mm} \textit{Step 3. Compute optimistic matching:}
    \FOR{$( m,k) \in [ M] \times [ K]$}
        \STATE $\tilde{\theta }_{mk,t}\leftarrow \arg\max_{\theta _{k} \in \hat{\Theta }_{k,t}} X_{m}^{\top } \theta _{k}$
        \STATE $\tilde{y}_{mk,t}\leftarrow X_{m,t}^{\top }\tilde{\theta }_{mk,t}$
    \ENDFOR
    \STATE $A_{t} \leftarrow \arg\max_{A\in \mathcal{A}_{M,K}}\sum \tilde{y}_{mA( m) ,t}$.
    \STATE \hspace{-3mm} \textit{Step 4. Observe rewards $y_{1,t} ,y_{2,t} ,...,y_{M,t}$}
    \STATE \hspace{-3mm} \textit{Step 5. Update parameters:}
    \FOR{$k\in A_{t}([ M])$}
        \STATE $V_{k,t}\leftarrow \sum _{s\in \Psi _{k,t}} x_{k,s} x_{k,s}^{\top } +\lambda I$ 
        \STATE $H_{k,t} \leftarrow \sum _{s\in \Psi _{k,t}} y_{k,s} x_{k,s}$
        % \STATE update $\hat{\Theta }_{k,t}$
    \ENDFOR
\ENDFOR
\end{algorithmic}
\end{algorithm}

\subsection{Constructing Confidence Intervals}
We use ridge regression to construct confidence intervals of $\theta _{k}^{*}$'s. Since the number of robots is greater than the number of humans, a robot may not be assigned to a human in some rounds. Therefore, for each robot, we need to keep track of the rounds where it was assigned to a human. For each $k$, at round $t$, let $\Psi _{k,t} =\left\{s\middle| s\in [ t-1] \ \text{and} \ k\in A_{s}([ M])\right\}$ be the set of rounds \textit{before} $t$ when robot $k$ was assigned to a human, $k_{s} :=A_{s}^{-1}( k)$ be the assigned human for $s\in \Psi _{k,t}$, where $A_{s}^{-1}$ is the inverse of $A_s$, and $( Y_{k_{s} ,s})_{s\in \Psi _{k,t}}$ and $( X_{k_{s} ,s})_{s\in \Psi _{k,t}}$ be the sets of rewards and human features. Using ridge regression, we obtain the estimated value of $\theta _{k}^{*}$ at the beginning of round $t$ as $\hat{\theta }_{k,t} =V_{k,t}^{-1} H_{k,t}$, where
\begin{equation}\label{eq:VH_defs}
V_{k,t} =\sum _{s\in \Psi _{k,t}} X_{k_{s} ,s} X_{k_{s} ,s}^{\top } +\lambda I , 
\quad
H_{k,t} =\sum _{s\in \Psi _{k,t}} Y_{k_{s} ,s} X_{k_{s} ,s} ,
\end{equation}
and $\lambda $ is a positive regularization coefficient. The confidence interval of $\theta _{k}^{*}$ at round $t$ is defined as 
\begin{equation}\label{eq:CI_def}
\hat{\Theta }_{k,t} :=\{\Vert \theta -\hat{\theta }_{k,t}\Vert _{V_{k,t}} \leqslant \rho _{k,t}\} ,
\end{equation}
where $\rho _{k,t} =\sigma \sqrt{2\log\frac{K}{\delta } +d\log\left( 1+\frac{| \Psi _{k,t}| L^{2}}{d\lambda }\right)} +\sqrt{\lambda } S$. Here the radius $\rho _{k,t}$ is designed to let the true parameter $\theta^*_k$ fall in $\hat{\Theta }_{k,t}$ with high probability (see Theorem~\ref{thm:CI_pro}).

\subsection{Optimistic Matching in the Face of Uncertainty}
After obtaining the ellipsoids $\hat{\Theta }_{k,t}$'s, {\tt LinMatch} computes the optimistic matching as 
\begin{equation}\label{eq:At_estimation}
A_{t} =\arg\max_{A\in \mathcal{A}_{M,K}}\sum _{m=1}^{M}\max_{\theta \in \hat{\Theta }_{A( m) ,t}} \theta ^{\top } X_{m,t} .
\end{equation}
This is equivalent to solving the following mixed integer program
\begin{equation*}
\begin{aligned}
\max \  & \sum _{m=1}^{M}\sum _{k=1}^{K} X_{m,t}^{\top } \theta _{k} 1_{\{A( m) =k\}}\\
\text{s.t. } & \theta _{k} \in \hat{\Theta }_{k,t} ,\ \text{for all } k\in [ K] ,\\
 & A\in \mathcal{A}_{M,K} ,
\end{aligned}
\end{equation*}
which can be rewritten using the incidence matrix notation \citep{conforti2014integer}:
\begin{equation}\label{eq:LP1}
\begin{aligned}
\max \quad  & c_{t}^{\prime }( e,\{\theta_k\}_{k=1}^K ) :=\sum _{m=1}^{M}\sum _{k=1}^{K} X_{m,t}^{\top } \theta _{k} e_{mk}\\
\text{s.t.} \quad  & I_{G} e\leqslant \mathbf{1} ,\\
 & e\in \{ 0,1\}^{MK} ,\\
 & \theta _{k} \in \hat{\Theta }_{k,t} ,\ \text{for all } k\in [ K] ,
\end{aligned} \tag{LP1}
\end{equation}
where $I_{G}$ is the incidence matrix of the bipartite graph $G:=[ M] \amalg [ K]$. We can view \eqref{eq:LP1} as a integer program parametrized by $\{\theta_k\}_{k=1}^K$. 
% (cite integer programming book theorem, see \ p145, Conforti)

To solve \eqref{eq:LP1}, we need to select $\theta _{k}$ from $\hat{\Theta }_{k,t}$ to maximize $c_{t}^{\prime }( e,\theta )$ for each matching $e$. We notice that $\sum _{m=1}^{M} e_{mk} \in \{0,1\}$, for all $k\in [ K]$, which implies that each $\theta _{k}$ will appear at most one non-zero summand in $c_{t}^{\prime }( e,\theta )$, and thus the cost function $c_{t}^{\prime }( e,\theta )$ can be replaced with 
\begin{equation*}
c_{t}( e) :=\sum _{m=1}^{M}\sum _{k=1}^{K}\max_{\theta _{k} \in \hat{\Theta }_{k,t}}\left( X_{m,t}^{\top } \theta _{k}\right) e_{mk}
\end{equation*}
without altering the optimal value and solution of \eqref{eq:LP1}. Therefore, by defining 
\begin{equation}\label{eq:theta_tilda}
\tilde{\theta }_{mk,t} =\arg\max_{\theta _{k} \in \hat{\Theta }_{k,t}} X_{m,t}^{\top } \theta _{k} \quad \text{and} \quad \tilde{y}_{mk,t} =X_{m,t}^{\top }\tilde{\theta }_{mk,t} ,
\end{equation}
we can convert \eqref{eq:LP1} to the following standard matching program
\begin{equation}
\begin{aligned}
\max \quad  & c_{t}( e) :=\sum _{m=1}^{M}\sum _{k=1}^{K}\tilde{y}_{mk,t} e_{mk}\\
\text{s.t.} \quad  & I_{G} e\leqslant \mathbf{1}\\
 & e\in \{0,1\}^{MK} ,
\end{aligned} \tag{LP2}
\end{equation}
which can be efficiently solved by the Hungarian algorithm in $O\left( K^{3}\right)$ \citep{conforti2014integer}.

%%%%%%%%%%%%%%%%%%%%%%%%%%%%%%%%%
\section{Theoretical Results}\label{sec:thm}
%%%%%%%%%%%%%%%%%%%%%%%%%%%%%%%%%
We analyze the {\tt LinMatch} algorithm and show that its regret is upper bounded by $O(d \sqrt{MKT})$. In addition, we provide a worst-case lower bound of $\Omega(\sqrt{dKT})$.
% , matching the upper bound with the $\sqrt{T}$ factor.

\subsection{Regret Upper Bound}\label{sec:upper_bound_regret}
The following result gives an upper bound on the regret of the {\tt LinMatch} algorithm.

% \begin{thm}\label{thm:upper}
% Assume $\left\Vert \theta _{k}^{*}\right\Vert \leqslant S$ for all $k\in [ K]$ and $\Vert X_{m,t}\Vert \leqslant L$ for all $( m,t) \in [ M] \times [ T]$. For $\delta $ and $\lambda $ that satisfy 
% \begin{equation*}
% % \begin{aligned}
% \delta \leqslant \frac{K}{\exp( d)} \ \text{ and} \quad
% \sqrt{\lambda } \geqslant \frac{\max_{m\in [M], k_{1} ,k_{2} \in [ K] ,s\in [ T]} X_{m,s}^{\top }\left( \theta _{k_{1}}^{*} -\theta _{k_{2}}^{*}\right)}{2S} ,
% % \end{aligned} \ \ 
% \end{equation*}
% with probability at least $1-\delta $, the regret of the {\tt LinMatch} algorithm satisfies, for all $t\in [ T]$,
% \begin{equation*}
% % \begin{aligned}
% R_{t} \leqslant  2\sqrt{2dtMK\log\left( 1+\frac{tML^{2}}{dK\lambda }\right)} 
% \times \left( \sigma \sqrt{2\log\frac{K}{\delta } +d\log\left( 1+\frac{tML^{2}}{dK\lambda }\right)} +\sqrt{\lambda } S\right).
% % \end{aligned}
% \end{equation*}
% \end{thm}

\begin{thm}\label{thm:upper}
Assume $\left\Vert \theta _{k}^{*}\right\Vert \leqslant S$ for all $k\in [ K]$ and $\Vert X_{m,t}\Vert \leqslant L$ for all $( m,t) \in [ M] \times [ T]$. For $\delta $ and $\lambda $ that satisfy 
\begin{equation}
\label{eq:delta_condition}
\delta \leqslant \frac{K}{\exp( d)}
\end{equation}
and
\begin{equation}
\label{eq:lambda_condition}
\sqrt{\lambda } \geqslant \frac{\max_{m\in [M], k_{1} ,k_{2} \in [ K] ,s\in [ T]} X_{m,s}^{\top }\left( \theta _{k_{1}}^{*} -\theta _{k_{2}}^{*}\right)}{2S} ,    
\end{equation}
with probability at least $1-\delta $, the regret of the {\tt LinMatch} algorithm satisfies, for all $t\in [ T]$,
\begin{equation*}
\begin{aligned}
R_{t} \leqslant  & 2\sqrt{2dtMK\log\left( 1+\frac{tML^{2}}{dK\lambda }\right)}  \times \\
& \left( \sigma \sqrt{2\log\frac{K}{\delta } +d\log\left( 1+\frac{tML^{2}}{dK\lambda }\right)} +\sqrt{\lambda } S\right).
\end{aligned}
\end{equation*}
\end{thm}

The proof can be found in the Appendix. The constraint on $\lambda$ can be replaced with a stronger condition that $\lambda\geq L^2$, which can be easier to check.

The difficulties of analyzing the upper bound of the {\tt LinMatch} algorithm arise from two aspects. First, the algorithm's future actions depend on the confidence intervals constructed using past actions, making the determination of the confidence set a complex task. Previous works in linear bandits have developed two approaches to solve this problem, either utilizing the theory of self-normalized processes \citep{pena2009self, abbasi2011improved} or expressing the estimated reward as a linear combination of rewards from arms that are independent of the current arm \citep{auer2002using, chu2011contextual}. In our study, we will use the theory of self-normalized processes to construct the confidence intervals of $\theta$'s, as it results in a simple form of algorithm. 

% \begin{figure}[t]
%     \centering
%     \includegraphics[width = 1\columnwidth]{figs/regret_fig.pdf}
%     \caption{The regret of an optimistic matching in the face of uncertainty. Blue: optimistic matching; red: optimal matching; black: actual matching.
%     }
%     \label{fig:regret}
% \end{figure}

Second, the combinatorial nature of the matching problem makes it a complicated work to express the regret in terms of the confidence intervals. To illustrate, consider the matching problem with two robots and two human agents. The time index $t$ is omitted for clarity. Suppose we have constructed the confidence intervals $\hat{\Theta}_1$ and $\hat{\Theta}_2$ and calculated the optimistic estimations $\hat{\theta}_{11}$, $\hat{\theta}_{12}$, $\hat{\theta}_{21}$, and $\hat{\theta}_{22}$ based on \eqref{eq:theta_tilda}. Assume the optimistic matching is to assign human 1 ($X_1$) to robot 1 ($\hat{\theta}_1$) and human 2 ($X_2$) to robot 2 ($\hat{\theta}_2$), while the actual optimal matching is to assign human 1 to robot 2 and human 2 to robot 1. Then the expected regret is
\begin{equation}\label{eq:regret_direct}
\begin{aligned}
R= & \underbrace{\left( X_{1}^{\top } \theta _{2}^{*} +X_{2}^{\top } \theta _{1}^{*}\right)}_{\text{optimal reward}} -\underbrace{\left( X_{1}^{\top } \theta _{1}^{*} +X_{2}^{\top } \theta _{2}^{*}\right)}_{\text{actual reward}} \\
= & X_{1}^{\top }\left( \theta _{2}^{*} -\theta _{1}^{*}\right) +X_{2}^{\top }\left( \theta _{1}^{*} -\theta _{2}^{*}\right).
\end{aligned}
\end{equation}

% \begin{equation}\label{eq:regret_direct}
% R= \underbrace{\left( X_{1}^{\top } \theta _{2}^{*} +X_{2}^{\top } \theta _{1}^{*}\right)}_{\text{optimal reward}} -\underbrace{\left( X_{1}^{\top } \theta _{1}^{*} +X_{2}^{\top } \theta _{2}^{*}\right)}_{\text{actual reward}}
% = X_{1}^{\top }\left( \theta _{2}^{*} -\theta _{1}^{*}\right) +X_{2}^{\top }\left( \theta _{1}^{*} -\theta _{2}^{*}\right).
% \end{equation}
Ideally, we want to prove the righthand side of \eqref{eq:regret_direct} will decrease over time. However, the differences between the unknown $\theta^*$'s are constant. To proceed, we observe that, if every true value $\theta^*_k$ is within its confidence interval $\hat{\Theta}_k$, then we have
\begin{equation}
\begin{aligned}
R 
&= \underbrace{\left( X_{1}^{\top } \theta _{2}^{*} +X_{2}^{\top } \theta _{1}^{*}\right)}_{\text{optimal reward}} -\underbrace{\left( X_{1}^{\top } \theta _{1}^{*} +X_{2}^{\top } \theta _{2}^{*}\right)}_{\text{actual reward}}\\
&\leqslant  \underbrace{\left( X_{1}^{\top }\hat{\theta }_{11} +X_{2}^{\top }\hat{\theta }_{22}\right)}_{\text{optimistic reward}} -\underbrace{\left( X_{1}^{\top } \theta _{1}^{*} +X_{2}^{\top } \theta _{2}^{*}\right)}_{\text{actual reward}}\\
& = X_{1}^{\top }\left(\hat{\theta }_{11} -\theta _{1}^{*}\right) +X_{2}^{\top }\left(\hat{\theta }_{22} -\theta _{2}^{*}\right)\\
& \leqslant  L\left(\left\Vert \hat{\theta }_{11} -\theta _{1}^{*}\right\Vert +\left\Vert \hat{\theta }_{22} -\theta _{2}^{*}\right\Vert \right),
\end{aligned}
\end{equation}
% \begin{equation}
% \begin{aligned}
% R= & \underbrace{\left( X_{1}^{\top } \theta _{2}^{*} +X_{2}^{\top } \theta _{1}^{*}\right)}_{\text{optimal reward}} -\underbrace{\left( X_{1}^{\top } \theta _{1}^{*} +X_{2}^{\top } \theta _{2}^{*}\right)}_{\text{actual reward}}
% \leqslant  \underbrace{\left( X_{1}^{\top }\hat{\theta }_{11} +X_{2}^{\top }\hat{\theta }_{22}\right)}_{\text{optimistic reward}} -\underbrace{\left( X_{1}^{\top } \theta _{1}^{*} +X_{2}^{\top } \theta _{2}^{*}\right)}_{\text{actual reward}}\\
% = & X_{1}^{\top }\left(\hat{\theta }_{11} -\theta _{1}^{*}\right) +X_{2}^{\top }\left(\hat{\theta }_{22} -\theta _{2}^{*}\right)
% \leqslant  L\left(\left\Vert \hat{\theta }_{11} -\theta _{1}^{*}\right\Vert +\left\Vert \hat{\theta }_{22} -\theta _{2}^{*}\right\Vert \right),
% \end{aligned}
% \end{equation}
where $L=\max\{\left\Vert X_{1}^{\top }\right\Vert ,\left\Vert X_{2}^{\top }\right\Vert\}$. This implies that we could use the size of the confidence intervals to bound the regret. It suffices to show that the confidence intervals in \eqref{eq:CI_def} shrink at a sublinear rate. The details can be found in the proof of Theorem \ref{thm:upper}.

Therefore, a critical step for upper bounding the regret is to understand how fast the confidence intervals $\hat{\Theta }_{k}$'s shrink when more rounds of data are available. For fixed $k\in [ K]$, we obtain
\begin{equation}
\begin{aligned}
\hat{\theta }_{k,t} 
& = V_{k,t}^{-1} H_{k,t} =V_{k,t}^{-1}\sum _{s\in \Psi _{k,t}} Y_{k_{s} ,s} X_{k_{s} ,s} \\ 
& = V_{k,t}^{-1}\sum _{s\in \Psi _{k,t}} X_{k_{s} ,s}\left( X_{k_{s} ,t}^{\top } \theta _{k}^{*} +\epsilon _{k_{s} ,t}\right)\\
& = V_{k,t}^{-1}\left( V_{k,t} \theta _{k}^{*} -\lambda \theta _{k}^{*} +\sum _{s\in \Psi _{k,t}} \epsilon _{k_{s} ,t} X_{k_{s} ,s}\right) \\
& = \theta _{k}^{*} -\lambda V_{k,t}^{-1} \theta _{k}^{*} +V_{k,t}^{-1} S_{k,t} ,
\end{aligned}
\end{equation}
% \begin{equation}
% \begin{aligned}
% \hat{\theta }_{k,t} = & V_{k,t}^{-1} H_{k,t} =V_{k,t}^{-1}\sum _{s\in \Psi _{k,t}} Y_{k_{s} ,s} X_{k_{s} ,s} = V_{k,t}^{-1}\sum _{s\in \Psi _{k,t}} X_{k_{s} ,s}\left( X_{k_{s} ,t}^{\top } \theta _{k}^{*} +\epsilon _{k_{s} ,t}\right)\\
% = & V_{k,t}^{-1}\left( V_{k,t} \theta _{k}^{*} -\lambda \theta _{k}^{*} +\sum _{s\in \Psi _{k,t}} \epsilon _{k_{s} ,t} X_{k_{s} ,s}\right) = \theta _{k}^{*} -\lambda V_{k,t}^{-1} \theta _{k}^{*} +V_{k,t}^{-1} S_{k,t} ,
% \end{aligned}
% \end{equation}
where we define $S_{k,t} =\sum _{s\in \Psi _{k,t}} \epsilon _{k_{s} ,t} X_{k_{s} ,s}$. Consequently, we can bound the difference between $\theta _{k}^{*}$ and $\hat{\theta }_{k,t}$ as
\begin{equation}\label{eq:theta_dist}
\begin{aligned}
& \left\Vert \hat{\theta }_{k,t} -\theta _{k}^{*}\right\Vert _{V_{k,t}} \\
= & \left\Vert V_{k,t}^{-1} S_{k,t} -\lambda V_{k,t}^{-1} \theta _{k}^{*}\right\Vert _{V_{k,t}} \\
\leqslant & \Vert S_{k,t}\Vert _{V_{k,t}^{-1}} +\lambda \left\Vert \theta _{k}^{*}\right\Vert _{V_{k,t}^{-1}}\\
\leqslant & \Vert S_{k,t}\Vert _{V_{k,t}^{-1}} +\lambda \left\Vert \theta _{k}^{*}\right\Vert _{( \lambda I)^{-1}} \\
= &\Vert S_{k,t}\Vert _{V_{k,t}^{-1}} +\sqrt{\lambda }\left\Vert \theta _{k}^{*}\right\Vert _{2} \\
\leqslant & \Vert S_{k,t}\Vert _{V_{k,t}^{-1}} +\sqrt{\lambda } S,
\end{aligned}
\end{equation}
% \begin{equation}\label{eq:theta_dist}
% \begin{aligned}
% \left\Vert \hat{\theta }_{k,t} -\theta _{k}^{*}\right\Vert _{V_{k,t}} & =\left\Vert V_{k,t}^{-1} S_{k,t} -\lambda V_{k,t}^{-1} \theta _{k}^{*}\right\Vert _{V_{k,t}} \leqslant \Vert S_{k,t}\Vert _{V_{k,t}^{-1}} +\lambda \left\Vert \theta _{k}^{*}\right\Vert _{V_{k,t}^{-1}}\\
%  & \leqslant \Vert S_{k,t}\Vert _{V_{k,t}^{-1}} +\lambda \left\Vert \theta _{k}^{*}\right\Vert _{( \lambda I)^{-1}}
% =\Vert S_{k,t}\Vert _{V_{k,t}^{-1}} +\sqrt{\lambda }\left\Vert \theta _{k}^{*}\right\Vert _{2}
% \leqslant \Vert S_{k,t}\Vert _{V_{k,t}^{-1}} +\sqrt{\lambda } S,
% \end{aligned}
% \end{equation}
where the first inequality comes from the triangle inequality and the second inequality holds true due to \eqref{eq:VH_defs}. Bounding $S_{k,t}$ is challenging, as actions in each round are not independent of each other. \citet{abbasi2011improved} employed a martingale argument to prove the following: for any $\delta _{k}  >0$, with probability at least $1-\delta _{k}$, for all $t >0$,
\begin{equation}\label{eq:S_size_k}
\begin{aligned}
& \Vert S_{k,t}\Vert _{\overline{V}_{k,t}^{-1}}^{2}
\leqslant 2\sigma ^{2}\log\left(\frac{(\det V_{k,t})^{1/2}}{\lambda ^{d/2} \delta _{k}}\right) \\
\leqslant & 2\sigma ^{2}\left(\log\frac{1}{\delta _{k}} +d\log\frac{d\lambda +| \Psi _{k,t}| L^{2}}{d\lambda }\right) .
\end{aligned}
\end{equation}
% \begin{equation}\label{eq:S_size_k}
% \Vert S_{k,t}\Vert _{\overline{V}_{k,t}^{-1}}^{2} \leqslant 2\sigma ^{2}\log\left(\frac{(\det V_{k,t})^{1/2}}{\lambda ^{d/2} \delta _{k}}\right)
% \leqslant 2\sigma ^{2}\left(\log\frac{1}{\delta _{k}} +d\log\frac{d\lambda +| \Psi _{k,t}| L^{2}}{d\lambda }\right) .
% \end{equation}
Note that \eqref{eq:S_size_k} only works for fixed $k$. Taking the union of \eqref{eq:S_size_k}'s for all $k$, we obtain that with probability at leaset $1-\sum _{k=1}^{K} \delta _{k}$, ~\eqref{eq:S_size_k} is true for all $k\in [ K]$ and $t\in [ T]$. Letting $\delta _{k} =\frac{\delta }{K}$, by combining \eqref{eq:CI_def}, \eqref{eq:theta_dist}, and \eqref{eq:S_size_k}, we obtain the following result on the tightness of the confidence intervals:
\begin{thm}
\label{thm:CI_pro}
For all $\delta  >0$,
$$
\mathbb{P}\left(\text{for all} \ ( k,t) \in [ K] \times [ T] ,\ \theta _{k}^{*} \in \hat{\Theta }_{k,t}\right) \geqslant 1-\delta .
$$
\end{thm}

The constraint in Theorem \ref{thm:upper} that $\delta \leqslant \frac{K}{\exp( d)}$ is necessary for deriving the union bound in \eqref{eq:union}, as the inequality in Lemma \ref{lema:inequality_union_bound} is sharp. By setting $\delta = \frac{K}{\exp( d)}$, we have the following corollary:

\begin{corollary}
\label{cor:upper}
Assume $\left\Vert \theta _{k}^{*}\right\Vert \leqslant S$ for all $k\in [ K]$ and $\Vert X_{m,t}\Vert \leqslant L$ for all $( m,t) \in [ M] \times [ T]$. For $\delta $ and $\lambda $ that satisfy $\delta =\frac{K}{\exp( d)}$ and $\sqrt{\lambda } \geqslant \frac{\max_{k_{1} ,k_{2} \in [ K] ,s\in [ T]} X_{m,s}^{\top }\left( \theta _{k_{1}}^{*} -\theta _{k_{2}}^{*}\right)}{2S}$, with probability at least $1-\frac{K}{\exp( d)}$, the regret of the {\tt LinMatch} algorithm satisfies
\begin{equation*}
R_{t} \leqslant C\sigma d\sqrt{tMK} ,\ \ \ \text{for all } t\in [ T] ,
\end{equation*}
where
% $C=\sqrt{8\log\left( 1+\frac{tML^{2}}{dK\lambda }\right)}\left(\sqrt{2+\log\left( 1+\frac{tML^{2}}{dK\lambda }\right)} +\sqrt{\frac{\lambda }{d}}\frac{S}{\sigma }\right)$.
\begin{equation*}
\begin{aligned}
C= & \sqrt{8\log\left( 1+\frac{tML^{2}}{dK\lambda }\right)} \times \\
 & \left(\sqrt{2+\log\left( 1+\frac{tML^{2}}{dK\lambda }\right)} +\sqrt{\frac{\lambda }{d}}\frac{S}{\sigma }\right).
\end{aligned}
\end{equation*}
\end{corollary}
This result implies that the algorithm has a sublinear upper bound of $O(d\sqrt{TMK})$ w.r.t. the number of total rounds $T$.

\subsection{ Regret Lower Bound}
We provide an instance-independent lower bound of the matching bandit problem with linear features.
\begin{thm}
\label{thm:lower_bound}
For the matching bandit problem described in Section \ref{sec:setting}, when $M \geq d$, there exists an matching bandit instance $\nu $ such that for any matching policy the regret satisfies
$
R_{\nu } \geqslant \frac{\sigma }{\sqrt{8}}\sqrt{dTK} .
$
\end{thm}

To show such a lower bound, we construct a family of similar bandit instances and lower bound the sum of their regret. The proof relies on the following variant of Pinsker's inequality:
\begin{lema}\label{lema:Pinsker}
For two probability measures $\mathbb{P}$ and $\mathbb{P}'$ on the same probability space $( \Omega ,\mathcal{F})$,
\begin{equation}\label{eq:pinsker}
\mathbb{P}'( A) \leqslant \mathbb{P}( A) +\sqrt{1-\exp\left( -\text{KL}\left( \mathbb{P},\mathbb{P}'\right)\right)} ,\ \ \text{for all }A\in \mathcal{F} ,
\end{equation}
where $\text{KL}\left( \mathbb{P},\mathbb{P}'\right)$ is the Kullback-Leibler (KL) divergence between $\mathbb{P}$ and $\mathbb{P}'$.
\end{lema}
% Its proof can be found in Section 14.3 of \cite{lattimore2020bandit}. 

% By reordering the terms in ~\eqref{eq:pinsker}, we have 
% \begin{equation*}
% \mathbb{P} (A)+\mathbb{P} '(A^{c} )\geqslant 1-\sqrt{1-\exp\left( -\text{KL}(\mathbb{P} ,\mathbb{P} ')\right)}.
% \end{equation*}
% which implies that, if $\text{KL}(\mathbb{P} ,\mathbb{P} ')$ is small, then the probability $\mathbb{P} (A)+\mathbb{P} '(A^{c} )$ will be high. 

% If we can construct several bandit instances on the same probability space where their pairwise KL divergence is small and 

\subsubsection{Decompose KL Divergence} 
To leverage Lemma \ref{lema:Pinsker}, we quantify the KL-divergence between two linear matching bandit instances. We first have to define the probability space and probability measure of the outcome of linear matching bandits. 
%A more detailed treatment of constructing the probability space of bandits can be found in Section 4.6 of \cite{lattimore2020bandit}. 
Second, we derive a decomposition property of the KL divergence of two linear matching bandits so we can express the total divergence as a weighted sum of divergence between different matching. We present the following result for decomposing the KL divergence and provide the details in Appendix \ref{app:ProbabilityMatching}.
% We first define probability measures on linear matching bandits. Consider the linear matching bandit $\nu =( \theta _{1} ,\dotsc ,\theta _{K} ,\mathcal{X}_{1} ,\dotsc ,\mathcal{X}_{t})$ defined in Section \ref{sec:setting}. Let $( A_{1} ,Y_{1} ,\dotsc ,A_{t} ,Y_{t})$ be the action-reward sequence produced by the interaction between $\nu $ and a policy $\pi =( \pi _{s})_{s=1}^{t}$. We define the probability space as $\Omega _{s} =\left(\mathcal{A}_{M,K} ,\mathbb{R}^{M}\right)^{s}$ and sigma algebra as $\mathcal{F}_{s} =\mathcal{B}( \Omega _{s})$ for $s\in [ t]$. A policy $\pi =( \pi _{s})_{s=1}^{t}$ is a sequence such that $\pi _{s}$ is a probability kernel from $( \Omega _{s-1} ,\mathcal{F}_{s-1})$ to $(\mathcal{A}_{M,K} ,\mathcal{B}(\mathcal{A}_{M,K}))$. Let $\mathbb{P}_{\nu }$ be a probability measure on $( \Omega _{t} ,\mathcal{F}_{t})$
% such that
% $\mathbb{P}_{\nu }( A_{s} | A_{1} ,Y_{1} ,\dotsc ,A_{s-1} ,Y_{s-1})=
% \pi _{s}( A_{s} | A_{1} ,Y_{1} ,\dotsc ,A_{s-1} ,Y_{s-1})
% $
% and 
% $
% \mathbb{P}_{\nu }( Y_{s} | A_{1} ,Y_{1} ,\dotsc ,A_{s-1} ,Y_{s-1} ,A_{s}) =P_{A_{s}}( Y_{s}) ,
% $
% where $P_{A_{s}}( Y_{s})$ is the conditional distribution of the reward $Y_{s}$ given the matching $A_{s}$. 

% Let $\nu ^{\prime } =\left( \theta _{1}^{\prime } ,\dotsc ,\theta _{K}^{\prime } ,\mathcal{X}_{1}^{\prime } ,\dotsc ,\mathcal{X}_{t}^{\prime }\right)$ be another linear matching bandit and $\mathbb{P}_{\nu ^{\prime }}$ be its associated probability measure. We can decompose $\operatorname{KL}(\mathbb{P}_{\nu } ,\mathbb{P}_{\nu ^{\prime }})$ as follows.

\begin{thm}\label{thm:decomposition_KL}
Let $\nu =( \theta _{1} ,\dotsc ,\theta _{K} ,\mathcal{X}_{1} ,\dotsc ,\mathcal{X}_{t})$ and $\nu ^{\prime } =\left( \theta _{1}^{\prime } ,\dotsc ,\theta _{K}^{\prime } ,\mathcal{X}_{1}^{\prime } ,\dotsc ,\mathcal{X}_{t}^{\prime }\right)$ be two linear matching bandits and $\mathbb{P}_{\nu }$ and $\mathbb{P}_{\nu ^{\prime }}$ be the corresponding probability measures under the same policy $\pi $. If $\operatorname{KL}\left(\mathbb{P} ,\mathbb{P}^{\prime }\right) < \infty $ and the noise $\{\epsilon _{m,s}\}_{m=1}^{M}$'s are independent of each other given matching $A_{s}$, then,
\begin{equation*}
\operatorname{KL}(\mathbb{P}_{\nu } ,\mathbb{P}_{\nu ^{\prime }}) =\sum _{s=1}^{t}\sum _{m=1}^{M}\mathbb{E}_{A_{s} \sim \nu \pi }\left[\operatorname{KL}\left( \mathbb{P}_{m,A_{s}} ,\mathbb{P}_{m,A_{s}}^{\prime }\right)\right] ,
\end{equation*}
where $\mathbb{P}_{m,A_{s}}$ is the distribution of $Y_{m,s}$ given $A_{s}$.
\end{thm}

When the noise $\epsilon _{m,s}$ is Gaussian with zero mean and $\sigma ^{2}$ variance, the distribution of $Y_{m,s}$ given $A_{s}$ is $N\left( X_{m,s}^{\top } \theta _{A_{s}( m)} ,\sigma ^{2}\right)$. As the KL divergence between two Gaussian distributions satisfies $\operatorname{KL}\left( N\left( \mu _{1} ,\sigma ^{2}\right) ,N\left( \mu _{2} ,\sigma ^{2}\right)\right) =\frac{( \mu _{1} -\mu _{2})^{2}}{2\sigma ^{2}}$, we can further decompose the KL-divergence in term of the reward difference in the two bandits as the following:

\begin{corollary}\label{coro:KL_decomp}
Let $\nu =( \theta _{1} ,\dotsc ,\theta _{K} ,\mathcal{X}_{1} ,\dotsc ,\mathcal{X}_{t})$ and $\nu ^{\prime } =\left( \theta _{1}^{\prime } ,\dotsc ,\theta _{K}^{\prime } ,\mathcal{X}_{1}^{\prime } ,\dotsc ,\mathcal{X}_{t}^{\prime }\right)$ be two linear matching bandits and $\mathbb{P}$ and $\mathbb{P}^{\prime }$ be the corresponding probability measures under the same policy $\pi $. If the noise $\{\epsilon _{m,s}\}_{m=1}^{M}$'s are independent Gaussian with zero mean and $\sigma ^{2}$ variance given matching $A_{s}$, then
\begin{equation*}
\begin{aligned}
 & \operatorname{KL}(\mathbb{P}_{\nu } ,\mathbb{P}_{\nu ^{\prime }})=\\
 & \frac{1}{2\sigma ^{2}}\sum _{m=1}^{M}\sum _{k=1}^{K}\sum _{s=1}^{t}\left(\left( X_{m,s}^{\top } \theta _{k} -X_{m,s}^{\prime \top } \theta _{k}^{\prime }\right)^{2}\mathbb{P}_{\nu }( A_{s}( m) =k)\right) ,
\end{aligned}
\end{equation*}
where $\mathbb{P}_{\nu }( A_{s}( m) =k)$ is the probability of assigning human $m$ robot $k$ at round $s$ under $\nu $.
\end{corollary}

With Lemma~\ref{lema:Pinsker} and Corollary~\ref{coro:KL_decomp}, we now can prove Theorem~\ref{thm:lower_bound} by constructing a set of bandit instances to show that any algorithm would suffer a $\Omega(\sqrt{dTK})$ regret on at least one bandit within this set. The proof is detailed in the Appendix.

\section{Numerical Experiments}
\subsection{Simulations Studies}
We tested the {\tt LinMatch} algorithm in simulations with $K=20$ robots and $M=10$ human agents. The noise is Gaussian with a variance of $\sigma^2=9$. The feature vectors $X_{m,t}$ and $\theta_k$ all satisfy Assumption \ref{assum:norm} with $S=L=10$, and they are generated uniformly from the cube $[-\frac{L}{\sqrt{d}},\frac{L}{\sqrt{d}}]^d$. The human agents' features are regenerated in each round while the robots' features are fixed throughout the experiment.

\begin{figure}[t]
  \centering
    
    \begin{subfigure}{1\columnwidth}
    \captionsetup{width=1\linewidth}
        \centering
        \includegraphics[width=0.9\columnwidth]{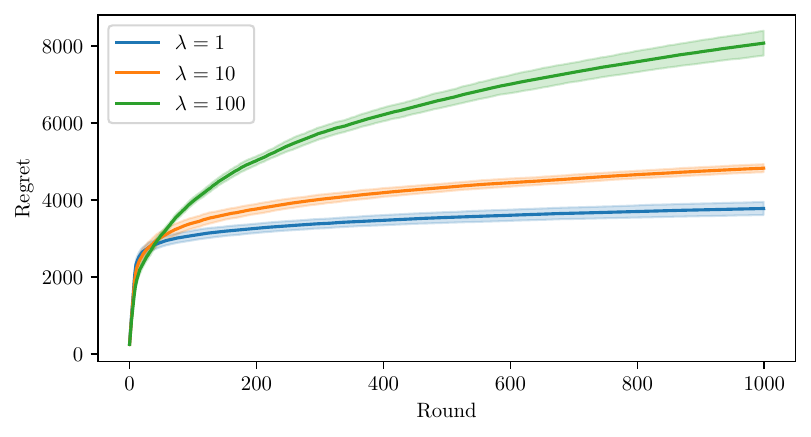}
        \caption{Regret $R$ against number of rounds $T$ with different values of $\lambda$.}
        \label{fig:RvsT}
    \end{subfigure}
    
    \begin{subfigure}{1\columnwidth}
    \captionsetup{width=1\linewidth}
        \centering
        \vspace{5pt}
        \includegraphics[width=0.9\columnwidth]{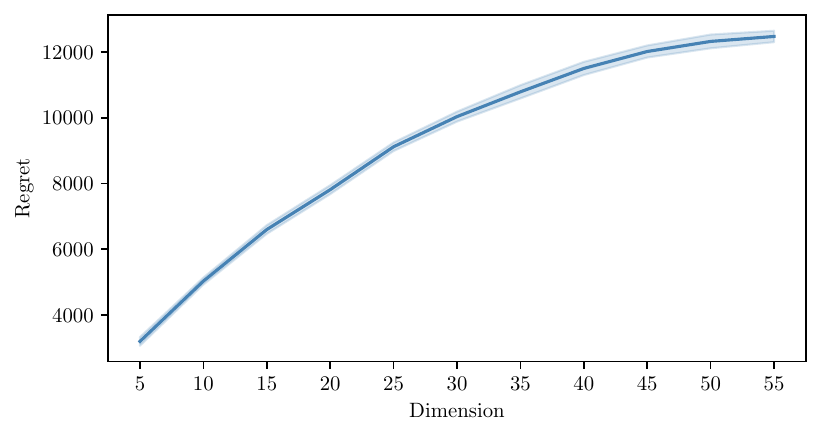}
        \caption{Regret $R$ against feature dimension $d$.}
        \label{fig:Rvsd}
    \end{subfigure}
    \caption{Experiment results with the number of robots $K=20$, number of humans $M=10$, and dimension of the feature vectors $d=5$. The shades represent the 95\% confidence intervals.}
  \vspace{-3mm}
\end{figure}

% \begin{figure}[ht]%
% \centering
% \subfigure[Regret $R$ against number of rounds $T$ with different values of $\lambda$.]{%
% \label{fig:RvsT}%
% \includegraphics[width=0.9\columnwidth]{figs/compare_with_lambda.pdf}}%
% \qquad
% \subfigure[Regret $R$ against feature dimension $d$.]{%
% \label{fig:Rvsd}%
% \includegraphics[width=0.9\columnwidth]{figs/compare_with_d.pdf}}%
% \caption{Experiment results with the number of robots $K=20$, number of humans $M=10$, and dimension of the feature vectors $d=5$. The shades represent the 95\% confidence intervals.}
% \vspace{-5mm}
% \end{figure}

\textbf{Varying regularization factor $\lambda$.} In the first setting, with $d=5$, we run the proposed matching algorithm for $T=1000$ rounds with 10 repetitions. The average regret is shown in Figure \ref{fig:RvsT}. As we mentioned in Section \ref{sec:upper_bound_regret}, a sufficient condition for Theorem \ref{thm:upper} to apply is that the regularization factor $\lambda$ is greater or equal to $L^2$, which translates to $\lambda\geq 100$ in this setting. However, the coefficient $C$ in Corollary \ref{cor:upper} suggests that the upper bound tends to increase when $\lambda$ increases. When we ran the experiment with different values of $\lambda$, we found that the average regret is smaller in the long run when $\lambda=1$. This finding suggests that in practical application, a lower value of $\lambda$ may lead to better outcomes, contrary to the sufficient condition of $\lambda\geq L^2$ for the upper bound guarantee.

\begin{figure}[t]
  \centering
    
    \begin{subfigure}{1\columnwidth}
    \captionsetup{width=1\linewidth}
        \centering
        \includegraphics[width=0.9\columnwidth]{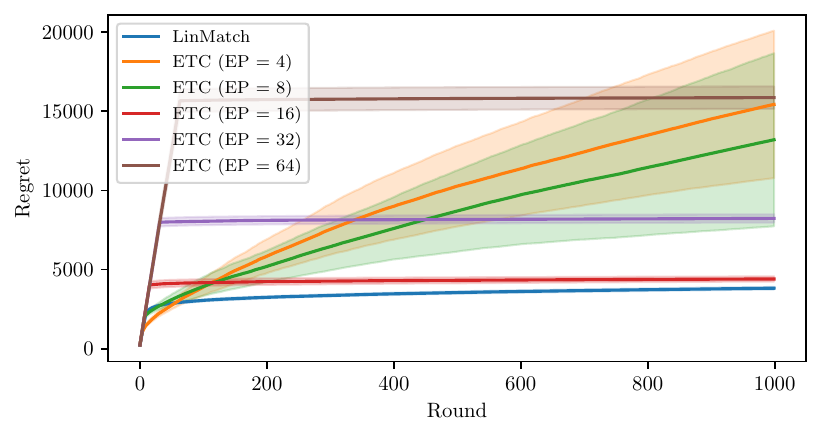}
        \caption{Comparing {\tt LinMatch} with ETC algorithm.}
        \label{fig:vsETC}
    \end{subfigure}
    
    \begin{subfigure}{1\columnwidth}
    \captionsetup{width=1\linewidth}
        \centering
        \vspace{5pt}
        \includegraphics[width=0.9\columnwidth]{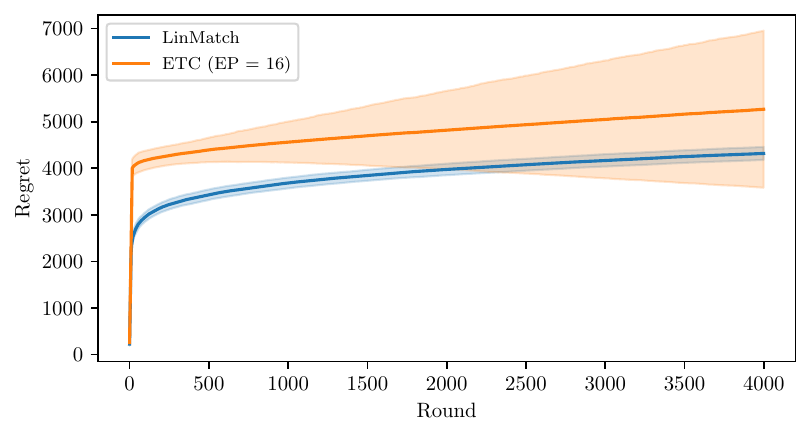}
        \caption{ETC algorithm has a higher variance in regret.}
        \label{fig:ETCVar}
    \end{subfigure}
    \caption{Comparing {\tt LinMatch} with ETC algorithm. EP indicates the length of the exploration phase. Solid lines are the mean values. Shades represent the 95\% confidence intervals.}
  \vspace{-3mm}
\end{figure}

% \begin{figure}[h]%
% \centering
% \subfigure[Comparing {\tt LinMatch} with ETC algorithm.]{%
% \label{fig:vsETC}%
% \includegraphics[width=0.9\columnwidth]{figs/compare_with_ETC_wide_range.pdf}}%
% \qquad
% \subfigure[ETC algorithm has a higher variance in regret.]{%
% \label{fig:ETCVar}%
% \includegraphics[width=0.9\columnwidth]{figs/compare_with_ETC_16_T=4000.pdf}}%
% \caption{Comparing {\tt LinMatch} with ETC algorithm. EP indicates the length of the exploration phase. Solid lines are the mean values. Shades represent the 95\% confidence intervals.}
% \end{figure}

\textbf{Varying dimension $d$.} In the second setting, we ran the algorithm with different feature dimensions, as shown in Figure~\ref{fig:Rvsd}. We repeated the experiment 10 times and calculated the average regret after 200 rounds ($T=200$). Results indicate that $R$ increases nonlinearly with $d$, which deviates from our upper bound. We speculate that the upper bound could be increased by a factor of $\sqrt{d}$.

\subsection{Comparing with explore-then-commit (ETC)} 
Since previous studies differ from our setting and thus cannot be directly applied, we compare our proposed algorithm with an explore-then-commit (ETC) algorithm. It is worth noting that \citet{li2022rate} proposed an ETC algorithm for the matching market; however, it cannot be directly applied to our scenario due to its exploration phase (EP) length being based on stable regret. Instead, we run the ETC algorithm with several different EP lengths, as shown in Figure~\ref{fig:vsETC}. Due to the discrete structure of matching problems, the ETC algorithm can achieve a sublinear regret if it can have an accurate estimation of the feature vectors after the exploration phase, though it suffers linear regret otherwise. We notice that the ETC algorithm shows similar performance to {\tt LinMatch} when EP $=16$ (see Figure~\ref{fig:ETCVar}). The 95\% confidence interval of ETC covers the curve of {\tt LinMatch}, indicating occasional superior performance. However, the narrower confidence interval of {\tt LinMatch} suggests more consistent performance.

\section{Conclusion}
We propose \texttt{LinMatch}, a UCB-like algorithm, to address the online multi-human multi-robot matching problem in a linear matching bandit framework. The optimistic matching problem in each round can be reformulated as a linear programming problem of maximum weighted matching and thus can be efficiently solved by the celebrated Hungarian algorithm. 

Our regret analysis shows that, for $T$ rounds, $M$ humans, $K$ robots, and $d$ dimensional feature vectors, the regret of {\tt LinMatch} is bounded by $O\left( d\sqrt{MKT}\right)$ and $\Omega\left(\sqrt{dKT}\right)$, which implies that the algorithm achieves sublinear regret and is rate-optimal in terms of the total number of rounds $T$. Numerical experiments show that, in the upper bound, the constraint on the regularization factor could be further relaxed to achieve lower regret and the $d$ factor could potentially be tightened to $\sqrt{d}$ to match the lower bound, which suggests directions for future study. 

In addition, the proposed algorithm utilizes the linear reward structure to recover the unknown feature vectors of the robots. Future studies can relax this assumption and assume more flexible forms of reward. For example, with a general reward structure, one can use the NeuralUCB algorithm \citep{zhou2020neural} to compute the upper confidence bounds of the matching rewards and then make the optimistic matching.

% =================
% =================
% =================
% =================
% =================
\subsubsection*{Acknowledgements}
This work is supported by the Air Force Office of Scientific Research under Grant No. FA9550-23-1-0044.
% All acknowledgments go at the end of the paper, including thanks to reviewers who gave useful comments, to colleagues who contributed to the ideas, and to funding agencies and corporate sponsors that provided financial support. 
% To preserve the anonymity, please include acknowledgments \emph{only} in the camera-ready papers. The acknowledgements do not count against the 9-page page limit in the camera-ready.

\bibliography{reference_list}

@inproceedings{erginbas2023interactive,
  title={Interactive Learning with Pricing for Optimal and Stable Allocations in Markets},
  author={Erginbas, Yigit Efe and Phade, Soham and Ramchandran, Kannan},
  booktitle={International Conference on Artificial Intelligence and Statistics},
  pages={9773--9806},
  year={2023},
  organization={PMLR}
}

@inproceedings{li2016contextual,
  title={Contextual combinatorial cascading bandits},
  author={Li, Shuai and Wang, Baoxiang and Zhang, Shengyu and Chen, Wei},
  booktitle={International conference on machine learning},
  pages={1245--1253},
  year={2016},
  organization={PMLR}
}

@INPROCEEDINGS{Wang2023,
  author={Wang, Ruiqi and Zhao, Dezhong and Min, Byung-Cheol},
  booktitle={2023 IEEE/RSJ International Conference on Intelligent Robots and Systems (IROS)}, 
  title={Initial Task Allocation for Multi-Human Multi-Robot Teams with Attention-Based Deep Reinforcement Learning}, 
  year={2023},
  volume={},
  number={},
  pages={7915-7922},
  doi={10.1109/IROS55552.2023.10341410}}

@article{Fu22,
    author = {Fu, Bo and Kathuria, Tribhi and Rizzo, Denise and Castanier, Matthew and Yang, X. Jessie and Ghaffari, Maani and Barton, Kira},
    title = "{Simultaneous Human-Robot Matching and Routing for Multi-Robot Tour Guiding Under Time Uncertainty}",
    journal = {Journal of Autonomous Vehicles and Systems},
    volume = {1},
    number = {4},
    pages = {041005},
    year = {2022},
    month = {02},
    issn = {2690-702X},
    doi = {10.1115/1.4053428},
    url = {https://doi.org/10.1115/1.4053428},
    eprint = {https://asmedigitalcollection.asme.org/autonomousvehicles/article-pdf/1/4/041005/6836042/javs\_1\_4\_041005.pdf},
}

@article{dahiya2022scalable,
  title={Scalable operator allocation for multirobot assistance: A restless bandit approach},
  author={Dahiya, Abhinav and Akbarzadeh, Nima and Mahajan, Aditya and Smith, Stephen L},
  journal={IEEE Transactions on Control of Network Systems},
  volume={9},
  number={3},
  pages={1397--1408},
  year={2022},
  publisher={IEEE}
}

@INPROCEEDINGS{JiRSS22, 
    AUTHOR    = {Tianchen Ji AND Roy Dong AND Katherine Driggs-Campbell}, 
    TITLE     = {{Traversing Supervisor Problem: An Approximately Optimal Approach to Multi-Robot Assistance}}, 
    BOOKTITLE = {Proceedings of Robotics: Science and Systems}, 
    YEAR      = {2022}, 
    ADDRESS   = {New York City, NY, USA}, 
    MONTH     = {June}, 
    DOI       = {10.15607/RSS.2022.XVIII.059} 
}

@inproceedings{lippi2023optimal,
  title={An optimal allocation and scheduling method in human-multi-robot precision agriculture settings},
  author={Lippi, Martina and Gallou, Jorand and Gasparri, Andrea and Marino, Alessandro},
  booktitle={2023 31st Mediterranean Conference on Control and Automation (MED)},
  pages={541--546},
  year={2023},
  organization={IEEE}
}

@book{conforti2014integer,
  title={Integer programming models},
  author={Conforti, Michele and Cornu{\'e}jols, G{\'e}rard and Zambelli, Giacomo and Conforti, Michele and Cornu{\'e}jols, G{\'e}rard and Zambelli, Giacomo},
  year={2014},
  publisher={Springer}
}

@inproceedings{zhou2020neural,
  title={Neural contextual bandits with ucb-based exploration},
  author={Zhou, Dongruo and Li, Lihong and Gu, Quanquan},
  booktitle={International Conference on Machine Learning},
  pages={11492--11502},
  year={2020},
  organization={PMLR}
}

@article{liu2021coordinating,
  title={Coordinating human-robot teams with dynamic and stochastic task proficiencies},
  author={Liu, Ruisen and Natarajan, Manisha and Gombolay, Matthew C},
  journal={ACM Transactions on Human-Robot Interaction (THRI)},
  volume={11},
  number={1},
  pages={1--42},
  year={2021},
  publisher={ACM New York, NY}
}

@inproceedings{gombolay2015coordination,
  title={Coordination of human-robot teaming with human task preferences},
  author={Gombolay, Matthew Craig and Huang, Cindy and Shah, Julie},
  booktitle={2015 AAAI Fall Symposium Series},
  year={2015}
}

@inproceedings{freedy2008multiagent,
  title={Multiagent Adjustable Autonomy Framework (MAAF) for multi-robot, multi-human teams},
  author={Freedy, Amos and Sert, Onur and Freedy, Elan and McDonough, James and Weltman, Gershon and Tambe, Milind and Gupta, Tapana and Grayson, William and Cabrera, Pedro},
  booktitle={2008 International Symposium on Collaborative Technologies and Systems},
  pages={498--505},
  year={2008},
  organization={IEEE}
}

@inproceedings{ramchurn2015study,
  title={A study of human-agent collaboration for multi-UAV task allocation in dynamic environments},
  author={Ramchurn, Sarvapali D and Fischer, Joel E and Ikuno, Yuki and Wu, Feng and Flann, Jack and Waldock, Antony},
  booktitle={Proceedings of the 24th International Joint Conference on Artificial Intelligence (IJCAI)},
  pages={1184--1192},
  year={2015}
}

@article{rusmevichientong2010linearly,
  title={Linearly parameterized bandits},
  author={Rusmevichientong, Paat and Tsitsiklis, John N},
  journal={Mathematics of Operations Research},
  volume={35},
  number={2},
  pages={395--411},
  year={2010},
  publisher={INFORMS}
}

@inproceedings{abe1999associative,
  title={Associative reinforcement learning using linear probabilistic concepts},
  author={Abe, Naoki and Long, Philip M},
  booktitle={ICML},
  pages={3--11},
  year={1999},
  organization={Citeseer}
}

@inproceedings{liu2018contextual,
  title={Contextual dependent click bandit algorithm for web recommendation},
  author={Liu, Weiwen and Li, Shuai and Zhang, Shengyu},
  booktitle={International Computing and Combinatorics Conference},
  pages={39--50},
  year={2018},
  organization={Springer}
}

@inproceedings{katariya2016dcm,
  title={DCM bandits: Learning to rank with multiple clicks},
  author={Katariya, Sumeet and Kveton, Branislav and Szepesvari, Csaba and Wen, Zheng},
  booktitle={International Conference on Machine Learning},
  pages={1215--1224},
  year={2016},
  organization={PMLR}
}

@inproceedings{kveton2015cascading,
  title={Cascading bandits: Learning to rank in the cascade model},
  author={Kveton, Branislav and Szepesvari, Csaba and Wen, Zheng and Ashkan, Azin},
  booktitle={International conference on machine learning},
  pages={767--776},
  year={2015},
  organization={PMLR}
}

@article{auer2002using,
  title={Using confidence bounds for exploitation-exploration trade-offs},
  author={Auer, Peter},
  journal={Journal of Machine Learning Research},
  volume={3},
  number={Nov},
  pages={397--422},
  year={2002}
}

@inproceedings{chu2011contextual,
  title={Contextual bandits with linear payoff functions},
  author={Chu, Wei and Li, Lihong and Reyzin, Lev and Schapire, Robert},
  booktitle={Proceedings of the Fourteenth International Conference on Artificial Intelligence and Statistics},
  pages={208--214},
  year={2011},
  organization={JMLR Workshop and Conference Proceedings}
}

@inproceedings{li2011unbiased,
  title={Unbiased offline evaluation of contextual-bandit-based news article recommendation algorithms},
  author={Li, Lihong and Chu, Wei and Langford, John and Wang, Xuanhui},
  booktitle={Proceedings of the fourth ACM international conference on Web search and data mining},
  pages={297--306},
  year={2011}
}

@article{jagadeesan2021learning,
  title={Learning equilibria in matching markets from bandit feedback},
  author={Jagadeesan, Meena and Wei, Alexander and Wang, Yixin and Jordan, Michael and Steinhardt, Jacob},
  journal={Advances in Neural Information Processing Systems},
  volume={34},
  pages={3323--3335},
  year={2021}
}

@inproceedings{cen2022regret,
  title={Regret, stability \& fairness in matching markets with bandit learners},
  author={Cen, Sarah H and Shah, Devavrat},
  booktitle={International Conference on Artificial Intelligence and Statistics},
  pages={8938--8968},
  year={2022},
  organization={PMLR}
}

@article{gale1962college,
  title={College admissions and the stability of marriage},
  author={Gale, David and Shapley, Lloyd S},
  journal={The American Mathematical Monthly},
  volume={69},
  number={1},
  pages={9--15},
  year={1962},
  publisher={Taylor \& Francis}
}

@article{roth2005pairwise,
  title={Pairwise kidney exchange},
  author={Roth, Alvin E and S{\"o}nmez, Tayfun and {\"U}nver, M Utku},
  journal={Journal of Economic theory},
  volume={125},
  number={2},
  pages={151--188},
  year={2005},
  publisher={Elsevier}
}

@article{roth1984evolution,
  title={The evolution of the labor market for medical interns and residents: a case study in game theory},
  author={Roth, Alvin E},
  journal={Journal of political Economy},
  volume={92},
  number={6},
  pages={991--1016},
  year={1984},
  publisher={The University of Chicago Press}
}

@inproceedings{liu2020competing,
  title={Competing bandits in matching markets},
  author={Liu, Lydia T and Mania, Horia and Jordan, Michael},
  booktitle={International Conference on Artificial Intelligence and Statistics},
  pages={1618--1628},
  year={2020},
  organization={PMLR}
}

@article{liu2021bandit,
  title={Bandit learning in decentralized matching markets},
  author={Liu, Lydia T and Ruan, Feng and Mania, Horia and Jordan, Michael I},
  journal={The Journal of Machine Learning Research},
  volume={22},
  number={1},
  pages={9612--9645},
  year={2021},
  publisher={JMLRORG}
}

@book{pena2009self,
  title={Self-normalized processes: Limit theory and Statistical Applications},
  author={Pe{\~n}a, Victor H and Lai, Tze Leung and Shao, Qi-Man},
  year={2009},
  publisher={Springer}
}

@book{lattimore2020bandit,
  title={Bandit algorithms},
  author={Lattimore, Tor and Szepesv{\'a}ri, Csaba},
  year={2020},
  publisher={Cambridge University Press}
}

@inproceedings{li2010contextual,
  title={A contextual-bandit approach to personalized news article recommendation},
  author={Li, Lihong and Chu, Wei and Langford, John and Schapire, Robert E},
  booktitle={Proceedings of the 19th international conference on World wide web},
  pages={661--670},
  year={2010}
}

@article{lattimore2018toprank,
  title={Toprank: A practical algorithm for online stochastic ranking},
  author={Lattimore, Tor and Kveton, Branislav and Li, Shuai and Szepesvari, Csaba},
  journal={Advances in Neural Information Processing Systems},
  volume={31},
  year={2018}
}

@inproceedings{le2006multi,
  title={Multi-agent task assignment in the bandit framework},
  author={Le Ny, Jerome and Dahleh, Munther and Feron, Eric},
  booktitle={Proceedings of the 45th IEEE Conference on Decision and Control},
  pages={5281--5286},
  year={2006},
  organization={IEEE}
}

@inproceedings{hassan2014multi,
  title={A multi-armed bandit approach to online spatial task assignment},
  author={Hassan, Umair Ul and Curry, Edward},
  booktitle={2014 IEEE 11th Intl Conf on Ubiquitous Intelligence and Computing and 2014 IEEE 11th Intl Conf on Autonomic and Trusted Computing and 2014 IEEE 14th Intl Conf on Scalable Computing and Communications and Its Associated Workshops},
  pages={212--219},
  year={2014},
  organization={IEEE}
}

@article{ul2016efficient,
  title={Efficient task assignment for spatial crowdsourcing: A combinatorial fractional optimization approach with semi-bandit learning},
  author={ul Hassan, Umair and Curry, Edward},
  journal={Expert Systems with Applications},
  volume={58},
  pages={36--56},
  year={2016},
  publisher={Elsevier}
}

@inproceedings{li2019online,
  title={Online learning to rank with features},
  author={Li, Shuai and Lattimore, Tor and Szepesv{\'a}ri, Csaba},
  booktitle={International Conference on Machine Learning},
  pages={3856--3865},
  year={2019},
  organization={PMLR}
}

@article{li2022rate,
  title={Rate-optimal contextual online matching bandit},
  author={Li, Yuantong and Wang, Chi-hua and Cheng, Guang and Sun, Will Wei},
  journal={arXiv preprint arXiv:2205.03699},
  year={2022}
}

@article{tewari2017ads,
  title={From ads to interventions: Contextual bandits in mobile health},
  author={Tewari, Ambuj and Murphy, Susan A},
  journal={Mobile health: sensors, analytic methods, and applications},
  pages={495--517},
  year={2017},
  publisher={Springer}
}

@article{abbasi2011improved,
  title={Improved algorithms for linear stochastic bandits},
  author={Abbasi-Yadkori, Yasin and P{\'a}l, D{\'a}vid and Szepesv{\'a}ri, Csaba},
  journal={Advances in neural information processing systems},
  volume={24},
  year={2011}
}

\clearpage
\appendix
\thispagestyle{empty}

% Supplementary material: To improve readability, you must use a single-column format for the supplementary material.
\onecolumn
\aistatstitle{Appendix to A Linear Matching Bandit Approach to \\ Online Multi-Human Multi-Robot Matching}

\section{Summary of Major Notation}\label{app:sec:notation}
The major notation is summarized in Table \ref{tab:notations}.
\renewcommand{\arraystretch}{1.2}% Tighter
\begin{table}[h!]
    \caption{Summary of Major Notation}
    \centering
    {\begin{tabular}{c|p{12cm}}
        \toprule 
        \textbf{Notation} & \textbf{Description} \\ \hline
        $ K $               & Number of robots. \\ \hline
        $ M $               & Number of human agents. \\ \hline
        $ T $               & Number of rounds. \\ \hline
        $ d $               & Dimension of the feature vectors. \\ \hline
        $\theta_k $         & Feature vector of robot $k$. \\ \hline
        $X_m$               & Feature vector of human agent $m$. \\ \hline
        $S$                 & Bound on the magnitude of the robot feature vectors. \\ \hline
        $L$                 & Bound on the magnitude of the human feature vectors. \\ \hline
        $A_t$ & Matching at round $ t $. \\ \hline
        $ \mathcal{A}_{M,K} $ & Set of all injection functions from $ [M] $ to $ [K] $; all possible matching from $ [M] $ to $ [K] $. \\ \hline
        $ Y_{m,t} $         & Reward for human agent $ m $ at round $ t $. \\ \hline
        $ \mathbf{Y}_{t} $  & Reward vector for all human agents at round $ t $. \\ \hline
        $ \mathcal{H}_{t-1} $ & History of features, actions, and rewards up to round $ t-1 $. \\ \hline
        $\epsilon_{m,t}$ & Random noise of the reward produced by human $m$ with their robot. \\ \hline
        $ \sigma $          & Noise $\epsilon_{m,t}$ are $\sigma$-sub Gaussian. \\ \hline
        $ {R}_{T} $         & Regret over $ T $ rounds. \\ \hline
        $ \hat{\theta}_{k,t} $ & Estimated feature vector of robot $ k $ at round $ t $. \\ \hline
        $ \hat{\Theta}_{k,t}$ & Confidence ellipsoid for the parameter vector $ \theta_{k} $ at round $ t $. \\ \hline
        $ \rho_{k,t} $      & Confidence interval radius for robot $ k $ at round $ t $. \\ \hline
        $ {V}_{k,t} $       & Covariance matrix for robot $ k $ at round $ t $. \\ \hline
        $ {H}_{k,t} $       & Summation of feature vectors weighted by rewards for robot $ k $ at round $ t $. \\ \hline
        $ \lambda $         & Regularization parameter. \\ \hline
        $ \delta $          & Probability parameter used in the confidence interval calculation. \\ \hline
        $ {\Psi}_{k,t} $    & Set of rounds before $ t $ when robot $ k $ was assigned to a human. \\ \hline
        $ {S}_{k,t} $       & Summation of noise-weighted feature vectors for robot $ k $ at round $ t $. \\ \hline
        $ {I}_{{G}} $       & Incidence matrix of bipartite graph $ {G} $. \\ \hline
        $ \tilde{y}_{mk,t} $ & Optimistic estimation of reward for matching human $ m $ and robot $ k $ at round $ t $. \\ \hline
        $\nu$       & Matching bandit instance. \\ \hline
        $ \mathbb{P}_{\nu} $   & Probability measure based on bandit instance $ \nu $. \\ 
        \bottomrule
    \end{tabular}}
    \label{tab:notations}
\end{table}

\section{Detailed Examples of Online Multi-Human Multi-Robot Teaming} 
\label{appsec:examples}

\subsection{Example 1: Automated Warehouse Management}

\textbf{Scenario.} In an automated warehouse, human operators and robots collaborate to manage inventory and fulfill orders. Human operators handle high-level decision-making tasks such as monitoring inventory levels and addressing issues that arise, while robots perform physical tasks such as picking, packing, and transporting items. The warehouse operates in a dynamic environment where the number and type of tasks can change frequently. Human operators have varying levels of skills and experience, which influence their efficiency in different tasks. Similarly, robots have diverse capabilities, speeds, and reliabilities that affect their performance. 

\noindent \textbf{Human-robot online teaming.} To optimize the overall efficiency of warehouse operations, it is crucial to dynamically match human operators with robots in real-time. This involves continuously assessing the skills of human operators and the capabilities of robots, and making optimal assignments that maximize productivity. By doing so, the warehouse can ensure that tasks are completed more quickly and accurately, reducing errors and improving overall efficiency.

\noindent \textbf{Outcome.} Dynamic teaming in this context leads to higher operational efficiency, fewer errors, and increased productivity, as it allows the warehouse to adapt to changing conditions and make the best use of available human and robotic resources.

\subsection{Example 2: Disaster Response Coordination}

\textbf{Scenario:} In a disaster response scenario, multiple human emergency responders and robots are deployed to assist with search and rescue operations, damage assessment, and delivery of supplies. The environment is highly dynamic and unpredictable, requiring efficient and adaptable team formations to respond effectively to changing conditions. Human responders have diverse expertise, including medical, engineering, and logistics, which are critical for various aspects of the response efforts. Robots, equipped with different functionalities such as drones for aerial surveillance and ground robots for debris removal, provide essential support in challenging and hazardous conditions.

\noindent \textbf{Human-robot online teaming.} The success of the disaster response efforts depends on the ability to dynamically assign human responders and robots to tasks in real-time. This involves continuously evaluating the skills and conditions of human responders, as well as the operational status and capabilities of robots. By optimally pairing human responders with robots based on real-time data, the response teams can maximize their effectiveness in saving lives, surveying areas, and delivering resources efficiently.

\noindent \textbf{Outcome.} Dynamic teaming in disaster response scenarios enhances coordination, improves efficiency, and increases the overall effectiveness of the response efforts. This leads to better outcomes in crisis situations, including more lives saved and quicker recovery.

\section{Technical Lemmas}
% \yao{The following is usually called the elliptical potential lemma. The proof can be found in TLCS 19.4. Discuss about the possibility using the generalized elliptical potential lemma and its impact.}
\begin{lema}[Lemma 11, \cite{abbasi2011improved}]\label{lema:elliptical_potential}
Let $\lambda  >0$ and $a_{1} ,\dotsc ,a_{n} \in \mathbb{R}^{d}$ be a sequence of vectors with $\Vert a_{s}\Vert _{2} \leq L< \infty $ for all $s\in [t],V_{s} =\lambda I+\sum _{q\leq s} a_{q} a_{q}^{\top }$. Then,
\begin{equation*}
\sum _{s=1}^{t}\left( 1\land \Vert a_{s}\Vert _{V_{s}^{-1}}^{2}\right) \leqslant 2d\log\left(\frac{d\lambda +tL^{2}}{d\lambda }\right)
\end{equation*}
\end{lema}

\textbf{Remark:} The condition in Theorem \ref{thm:upper} that $\max_{m,k_1,k_2} X_{m,t}^{\top }\left( \theta _{k_1}^{*} -\theta _{k_2}^{*}\right) \leqslant 2\sqrt{\lambda } S$ is to make sure we can use Lemma \ref{lema:elliptical_potential} to bound $r_s$ in eq. \eqref{eq:rs3}.

% We know $\rho \geqslant \sqrt{\lambda } S$. To use Lemma \ref{lema:elliptical_potential}, we have to show that
% \begin{equation*}
% r_{s} \leqslant 2\sum _{m=1}^{M}( 1\land \Vert X_{m,s}\Vert _{V_{A_{s}( m) ,s}^{-1}}) \rho _{A_{s}( m) ,s} .
% \end{equation*}
% We have shown that $r_{s} \leqslant 2\sum _{m=1}^{M}\Vert X_{m,s}\Vert _{V_{A_{s}( m) ,s}^{-1}} \rho _{A_{s}( m) ,s}$, thus we only need to show $r_{s} \leqslant 2\sum _{m=1}^{M} \rho _{A_{s}( m) ,s}$.
% From the definition, we know that $\rho \geqslant \sqrt{\lambda } S$ so it suffices to have $r_{s} \leqslant 2M\sqrt{\lambda } S$, which can be implied by $\max_{m,k_1,k_2} X_{m,t}^{\top }\left( \theta _{k_1}^{*} -\theta _{k_2}^{*}\right) \leqslant 2\sqrt{\lambda } S$.

\begin{lema}\label{lema:inequality_union_bound}
Let $\sigma $, a, $b$, $c$, $d$ be positive real numbers, $K$ be a positive integer, and $e_{j}$, $j\in [ K]$, be non-negative integers that satisfy $\sum _{k=1}^{K} e_{k} =M$. Given $\frac{a}{d} \geqslant 2$, the following is true:
\begin{equation}
\begin{aligned}
 & 2\sum _{k=1}^{K}\left( \sigma \sqrt{a+d\log( 1+be_{k})} +c\right)\sqrt{2e_{k} d\log( 1+be_{k})}\\
\leqslant  & 2\left( \sigma \sqrt{a+d\log\left( 1+b\frac{M}{K}\right)} +c\right)\sqrt{2MKd\log\left( 1+b\frac{M}{K}\right)} .
\end{aligned}
\end{equation}
\end{lema}

\begin{proof}
Let $l=\frac{a}{d} \geqslant 2$. Define functions $g_{1} :( 0,\infty )\rightarrow ( 0,\infty )$ and $g_{2} :( 0,\infty )\rightarrow ( 0,\infty )$ as 
\begin{equation*}
\begin{aligned}
g_{1}( x) = & \sqrt{x}\sqrt{l+\log( 1+x)}\sqrt{\log( 1+x)}\\
\text{and} \ g_{2}( x) = & \sqrt{x\log( 1+x)} .
\end{aligned}
\end{equation*}
We show they are both concave. Calculation shows that
\begin{equation*}
\begin{aligned}
 & \kappa ( x) g'' _{1}( x)\\
= & -2l\left( l+\frac{h}{x} -3\right) xh^{2} -4( l-1) xh^{3} -l^{2}( x-h)^{2}\\
 & -l^{2} x^{2} h^{2} -2lx^{2} h^{3} -x^{2} h^{4} -2xh^{4} -h^{4}\\
\leqslant  & -2l\left( l+\left( 1-\frac{x}{2}\right) -3\right) xh^{2} -4( l-1) xh^{3} -l^{2}( x-h)^{2}\\
 & -l^{2} x^{2} h^{2} -2lx^{2} h^{3} -x^{2} h^{4} -2xh^{4} -h^{4}\\
= & -2l( l+1-3) xh^{2} -l( l -1) x^{2} h^{2} -4( l-1) xh^{3} -l^{2}( x-h)^{2}\\
 & -2lx^{2} h^{3} -x^{2} h^{4} -2xh^{4} -h^{4}\\
\leqslant  & 0,
\end{aligned}
\end{equation*}
where $\kappa ( x) =4x^{3/2}\log( x+1)^{3/2}( l+\log( x+1))^{3/2}( x+1)^{2}$ and $h=\log( 1+x)$. In the equation above, the first equality is true because that $\frac{\log( 1+x)}{x} \geqslant 1-\frac{x}{2}$ and the second equality is true because all terms are negative. Therefore, $g_{1}( x)$ is a concave function. Also, we have
\begin{equation*}
g''_{2}( x) =-\frac{\frac{x}{( x+1)^{2}} -\frac{2}{x+1}}{2\ \sqrt{x\ \log( x+1)}} -\frac{\left(\log( x+1) +\frac{x}{x+1}\right)^{2}}{4\ ( x\ \log( x+1))^{3/2}} \leqslant 0.
\end{equation*}
Therefore, $g_{1}$ and $g_{2}$ are both concave.

Define function $G:( 0,\infty )\rightarrow ( 0,\infty )$ as
\begin{equation*}
G( x) =\left( \sigma \sqrt{a+d\log( 1+bx)} +c\right)\sqrt{2xd\log( 1+bx)} .
\end{equation*}
We obtain
\begin{equation*}
\begin{aligned}
G( x) = & \sqrt{\frac{2}{b}} \sigma d\sqrt{\frac{a}{d} +\log( 1+bx)}\sqrt{bx}\sqrt{\log( 1+bx)} +c\sqrt{\frac{2d}{b}}\sqrt{bx\log( 1+bx)}\\
= & \sqrt{\frac{2}{b}} \sigma dg_{1}( bx) +c\sqrt{\frac{2d}{b}} g_{2}( bx) .
\end{aligned}
\end{equation*}
As the composition of a linear function and a concave function is concave and the linear combination of concave functions is concave, we obtain that $G( x)$ is concave.

Therefore, we have
\begin{equation*}
\sum _{k=1}^{K} G( e_{i}) \leqslant KG\left(\frac{\sum _{k=1}^{K} e_{k}}{K}\right) =KG\left(\frac{M}{K}\right) ,
\end{equation*}
 which completes the proof.
    
\end{proof}

\section{Missing Technical Proofs}

\subsection{Proof of Theorem \ref{thm:upper}}\label{sec:upper_proof}
\begin{proof}
Fix $\delta  >0$. Conditional on $G:=\left\{\theta _{k}^{*} \in \hat{\Theta }_{k,t} \ \text{for all} \ ( k,t) \in [ K] \times [ T]\right\}$
, which happens with a probability at least $1-\delta $ by Theorem \ref{thm:CI_pro}, the instantaneous regret $r_{s}$ that occurs at round $s$ is 
\begin{equation}\label{eq:rs}
\begin{aligned}
r_{s} = & \sum _{m=1}^{M} X_{m,s}^{\top }\left( \theta _{A_{s}^{*}( m)}^{*} -\theta _{A_{s}( m)}^{*}\right)\\
= & \sum _{m=1}^{M} X_{m,s}^{\top }\left( \theta _{A_{s}^{*}( m)}^{*} -\tilde{\theta }_{mA_{s}( m) ,s} +\tilde{\theta }_{mA_{s}( m) ,s} -\theta _{A_{s}( m)}^{*}\right)\\
= & \sum _{m=1}^{M} X_{m,s}^{\top }\left( \theta _{A_{s}^{*}( m)}^{*} -\tilde{\theta }_{mA_{s}( m) ,s}\right) + \sum _{m=1}^{M} X_{m,s}^{\top }\left(\tilde{\theta }_{mA_{s}( m) ,s} -\theta _{A_{s}( m)}^{*}\right) .
\end{aligned}
\end{equation}

By the definitions of $A_{s}$ and $\tilde{\theta }_{mk,s}$ from  \eqref{eq:At_estimation} and \eqref{eq:theta_tilda}, 
% \begin{equation*}
% \begin{aligned}
%  & \sum _{m=1}^{M} X_{m,s}^{\top }\tilde{\theta }_{mA_{s}( m) ,s} =\max_{A\in \mathcal{A}}\sum _{m=1}^{M}\max_{\theta \in \hat{\Theta }_{A( m) ,s}} X_{m,s}^{\top } \theta \\
% \geqslant  & \max_{A\in \mathcal{A}}\sum _{m=1}^{M} X_{m,s}^{\top } \theta _{A( m)}^{*} \geqslant \sum _{m=1}^{M} X_{m,s}^{\top } \theta _{A_{s}^{*}( m)}^{*} ,
% \end{aligned}
% \end{equation*}
\begin{equation*}
\sum _{m=1}^{M} X_{m,s}^{\top }\tilde{\theta }_{mA_{s}( m) ,s} =\max_{A\in \mathcal{A}}\sum _{m=1}^{M}\max_{\theta \in \hat{\Theta }_{A( m) ,s}} X_{m,s}^{\top } \theta \geqslant \max_{A\in \mathcal{A}}\sum _{m=1}^{M} X_{m,s}^{\top } \theta _{A( m)}^{*} \geqslant \sum _{m=1}^{M} X_{m,s}^{\top } \theta _{A_{s}^{*}( m)}^{*} ,
\end{equation*}
where the first inequality is true because event $G$ happens. Therefore, the first summation in the last line of  \eqref{eq:rs} is non-positive, and thus
\begin{equation}\label{eq:rs2}
\begin{aligned}
r_{s} \leqslant  & \sum _{m=1}^{M} X_{m,s}^{\top }\left(\tilde{\theta }_{mA_{s}( m) ,s} -\theta _{A_{s}( m)}^{*}\right)\\
= & \sum _{m=1}^{M} X_{m,s}^{\top }(\left(\tilde{\theta }_{mA_{s}( m) ,s} -\hat{\theta }_{A_{s}( m) ,t}\right) +\left(\hat{\theta }_{A_{s}( m) ,t} -\theta _{A_{s}( m)}^{*}\right))\\
\leqslant  & \sum _{m=1}^{M}\Vert X_{m,s}\Vert _{V_{A_{s}( m) ,s}^{-1}}(\left\Vert \tilde{\theta }_{mA_{s}( m) ,s} -\hat{\theta }_{A_{s}( m) ,t}\right\Vert _{V_{A_{s}( m) ,s}} +\left\Vert \hat{\theta }_{A_{s}( m) ,t} -\theta _{A_{s}( m)}^{*}\right\Vert _{V_{A_{s}( m) ,s}})\\
\leqslant  & 2\sum _{m=1}^{M}\Vert X_{m,s}\Vert _{V_{A_{s}( m) ,s}^{-1}} \rho _{A_{s}( m) ,s} .
\end{aligned}
\end{equation}
On the other hand, by our assumption that $\sqrt{\lambda } \geqslant \frac{\max_{k_{1} ,k_{2} \in [ K] ,s\in [ T]} X_{m,s}^{\top }\left( \theta _{k_{1}}^{*} -\theta _{k_{2}}^{*}\right)}{2S}$ and that $\rho _{A_{s}( m) ,s} \geqslant \sqrt{\lambda } S$, we obtain
% \begin{equation}\label{eq:rs3}
% \begin{aligned}
% r_{s} = & \sum _{m=1}^{M} X_{m,s}^{\top }\left( \theta _{A_{s}^{*}( m)}^{*} -\theta _{A_{s}( m)}^{*}\right)\\
% \leqslant  & \sum _{m=1}^{M}\max_{k_{1} ,k_{2} \in [ K] ,s\in [ T]} X_{m,s}^{\top }\left( \theta _{k_{1}}^{*} -\theta _{k_{2}}^{*}\right)\\
% \leqslant  & 2M\sqrt{\lambda } S
% \leqslant   2M\rho _{A_{s}( m) ,s} .
% \end{aligned}
% \end{equation}
\begin{equation}\label{eq:rs3}
r_{s} = \sum _{m=1}^{M} X_{m,s}^{\top }\left( \theta _{A_{s}^{*}( m)}^{*} -\theta _{A_{s}( m)}^{*}\right)  
\leqslant  \sum _{m=1}^{M}\max_{k_{1} ,k_{2} \in [ K] ,s\in [ T]} X_{m,s}^{\top }\left( \theta _{k_{1}}^{*} -\theta _{k_{2}}^{*}\right)
\leqslant  2M\sqrt{\lambda } S
\leqslant   2M\rho _{A_{s}( m) ,s}.
\end{equation}
 By  \eqref{eq:rs2} and \eqref{eq:rs3}, we get 
$
r_{s} \leqslant 2\sum _{m=1}^{M}( 1\land \Vert X_{m,s}\Vert _{V_{A_{s}( m) ,s}^{-1}}) \rho _{A_{s}( m) ,s} .
$
Therefore,
\begin{equation}
\begin{aligned}
\sum _{s=1}^{t} r_{s} & \leqslant 2\sum _{m=1}^{M}\sum _{s=1}^{t}( 1\land \Vert X_{m,s}\Vert _{V_{A_{s}( m) ,s}^{-1}}) \rho _{A_{s}( m) ,s}\\
 & =2\sum _{k=1}^{K}\sum _{s\in \Psi _{k,t}}( 1\land \Vert X_{k_{s} ,s}\Vert _{V_{k,s}^{-1}}) \rho _{k,s}\\
 & \leqslant 2\sum _{k=1}^{K} \rho _{k,t}\sum _{s\in \Psi _{k,t}}( 1\land \Vert X_{k_{s} ,s}\Vert _{V_{k,s}^{-1}})\\
 & \leqslant 2\sum _{k=1}^{K} \rho _{k,t}\sqrt{| \Psi _{k,t}| \sum _{s\in \Psi _{k,t}}\left( 1\land \Vert X_{k_{s} ,s}\Vert _{V_{k,s}^{-1}}^{2}\right)} .
\end{aligned}
\end{equation}
By Lemma \ref{lema:elliptical_potential}, 
\begin{equation}
\sum _{s\in \Psi _{k,t}}\left( 1\land \Vert X_{k_{s} ,s}\Vert _{V_{k,s}^{-1}}^{2}\right) \leqslant 2d\log\left( 1+\frac{| \Psi _{k,t}| L^{2}}{d\lambda }\right) ,
\end{equation}
and thus, by Lemma \ref{lema:inequality_union_bound}, the fact that $\sum _{k=1}^{K}| \Psi _{k,t}| =tM$, and the condition $\delta \leqslant \frac{K}{\exp( d)}$, we have 
% \begin{equation}\label{eq:union}
% \begin{aligned}
% \sum _{s=1}^{t} r_{s} \leqslant  & 2\sum _{k=1}^{K} \rho _{k,t}\sqrt{| \Psi _{k,t}| 2d\log\left( 1+\frac{| \Psi _{k,t}| L^{2}}{d\lambda }\right)}\\
% \leqslant  & 2K\left( \sigma \sqrt{2\log\frac{K}{\delta } +d\log\left( 1+\frac{tML^{2}}{dK\lambda }\right)} +\sqrt{\lambda } S\right)\\
%  & \times \sqrt{\frac{tM}{K} 2d\log\left( 1+\frac{tML^{2}}{dK\lambda }\right)} .
% \end{aligned}
% \end{equation}
\begin{equation}\label{eq:union}
\begin{aligned}
\sum _{s=1}^{t} r_{s} \leqslant  & 2\sum _{k=1}^{K} \rho _{k,t}\sqrt{| \Psi _{k,t}| 2d\log\left( 1+\frac{| \Psi _{k,t}| L^{2}}{d\lambda }\right)}\\
\leqslant  & 2K\left( \sigma \sqrt{2\log\frac{K}{\delta } +d\log\left( 1+\frac{tML^{2}}{dK\lambda }\right)} +\sqrt{\lambda } S\right) \times \sqrt{\frac{tM}{K} 2d\log\left( 1+\frac{tML^{2}}{dK\lambda }\right)} .
\end{aligned}
\end{equation}
\end{proof}

\subsection{Proof of Theorem \ref{thm:decomposition_KL}}
\label{app:ProbabilityMatching}
    
We first define probability measures on linear matching bandits. Consider the linear matching bandit $\nu =( \theta _{1} ,\dotsc ,\theta _{K} ,\mathcal{X}_{1} ,\dotsc ,\mathcal{X}_{t})$ defined in Section \ref{sec:setting}. Let $( A_{1} ,Y_{1} ,\dotsc ,A_{t} ,Y_{t})$ be the action-reward sequence produced by the interaction between $\nu $ and a policy $\pi =( \pi _{s})_{s=1}^{t}$. We define the probability space as $\Omega _{s} =\left(\mathcal{A}_{M,K} ,\mathbb{R}^{M}\right)^{s}$ and sigma algebra as $\mathcal{F}_{s} =\mathcal{B}( \Omega _{s})$ for $s\in [ t]$. A policy $\pi =( \pi _{s})_{s=1}^{t}$ is a sequence such that $\pi _{s}$ is a probability kernel from $( \Omega _{s-1} ,\mathcal{F}_{s-1})$ to $(\mathcal{A}_{M,K} ,\mathcal{B}(\mathcal{A}_{M,K}))$. The probability measure $\mathbb{P}_{\nu }$ on $( \Omega _{t} ,\mathcal{F}_{t})$ satisfies
$
\mathbb{P}_{\nu }( A_{s} | A_{1} ,Y_{1} ,\dotsc ,A_{s-1} ,Y_{s-1})=
\pi _{s}( A_{s} | A_{1} ,Y_{1} ,\dotsc ,A_{s-1} ,Y_{s-1})
$
and 
$
\mathbb{P}_{\nu }( Y_{s} | A_{1} ,Y_{1} ,\dotsc ,A_{s-1} ,Y_{s-1} ,A_{s}) =P_{A_{s}}( Y_{s}) ,
$
where $P_{A_{s}}( Y_{s})$ is the conditional distribution of the reward $Y_{s}$ given the matching $A_{s}$.

\begin{proof}
By its definition, the probability density function $p$ of $\mathbb{P}$ satisfies
% \begin{equation*}
% \begin{aligned}
%  & p( a_{1} ,y_{1} ,\dotsc ,a_{t} ,y_{t})\\
% = & \prod _{s=1}^{t} p( a_{s} | a_{1} ,y_{1} ,\dotsc ,y_{s-1}) p( y_{s} | a_{1} ,y_{1} ,\dotsc ,a_{s})\\
% = & \prod _{s=1}^{t} \pi _{s}( a_{s} | a_{1} ,y_{1} ,\dotsc ,y_{s-1}) p_{a_{s}}( y_{s}) .
% \end{aligned}
% \end{equation*}
\begin{equation*}
p( a_{1} ,y_{1} ,\dotsc ,a_{t} ,y_{t})
= \prod _{s=1}^{t} p( a_{s} | a_{1} ,y_{1} ,\dotsc ,y_{s-1}) p( y_{s} | a_{1} ,y_{1} ,\dotsc ,a_{s})
= \prod _{s=1}^{t} \pi _{s}( a_{s} | a_{1} ,y_{1} ,\dotsc ,y_{s-1}) p_{a_{s}}( y_{s}) .
\end{equation*}
Because $\operatorname{KL}(\mathbb{P}_{\nu } ,\mathbb{P}_{\nu ^{\prime }}) < \infty $, $\mathbb{P}_{\nu }$ is absolutely continuous w.r.t. $\mathbb{P}_{\nu ^{\prime }}$, and thus $\frac{p}{p^{\prime }}( A_{1} ,,\dotsc ,A_{t} ,Y_{t})$ is well defined. Therefore,
\begin{equation*}
\begin{aligned}
 & \operatorname{KL}(\mathbb{P}_{\nu } ,\mathbb{P}_{\nu ^{\prime }})
= \mathbb{E}_{\nu \pi }\left[\log\frac{p}{p^{\prime }}( A_{1} ,,\dotsc ,A_{t} ,Y_{t})\right]\\
= & \mathbb{E}_{\nu \pi }\left[\log\frac{\prod _{s=1}^{t} \pi _{s}( A_{s} | A_{1} ,Y_{1} ,\dotsc ,Y_{s-1}) p_{A_{s}}( Y_{s})}{\prod _{s=1}^{t} \pi _{s}( A_{s} | A_{1} ,Y_{1} ,\dotsc ,Y_{s-1}) p_{A_{s}}^{\prime }( Y_{s})}\right]
= \mathbb{E}_{\nu \pi }\left[\log\frac{\prod _{s=1}^{t}\prod _{m=1}^{M} p_{m,A_{s}}( Y_{m,s})}{\prod _{s=1}^{t}\prod _{m=1}^{M} p_{m,A_{s}}^{\prime }( Y_{m,s})}\right]\\
= & \mathbb{E}_{A_{s} \sim \nu \pi }\left[\mathbb{E}_{\nu \pi }\left[\log\frac{\prod _{s=1}^{t}\prod _{m=1}^{M} p_{m,A_{s}}( Y_{m,s})}{\prod _{s=1}^{t}\prod _{m=1}^{M} p_{m,A_{s}}^{\prime }( Y_{m,s})}\middle| A_{s}\right]\right]\\
= & \sum _{s=1}^{t}\sum _{m=1}^{M}\mathbb{E}_{A_{s} \sim \nu \pi }\left[\operatorname{KL}\left( P_{m,A_{s}} ,P_{m,A_{s}}^{\prime }\right)\right]
\end{aligned}
\end{equation*}
\end{proof}

\subsection{Proof of Corollary \ref{coro:KL_decomp}}

\allowdisplaybreaks
\begin{proof}
\begin{equation*}
\begin{aligned}
 & \operatorname{KL}(\mathbb{P}_{\nu } ,\mathbb{P}_{\nu ^{\prime }})
=  \sum _{s=1}^{t}\sum _{m=1}^{M}\mathbb{E}_{A_{s} \sim \nu \pi }\left[\operatorname{KL}\left( P_{m,A_{s}} ,P_{m,A_{s}}^{\prime }\right)\right]\\
= & \sum _{s=1}^{t}\sum _{m=1}^{M}\mathbb{E}_{A_{s} \sim \nu \pi }\left[\frac{\left( X_{m,s}^{\top } \theta _{A_{s}( m)} -X_{m,s}^{\prime \top } \theta _{A_{s}( m)}^{\prime }\right)^{2}}{2\sigma ^{2}}\right]\\
= & \frac{1}{2\sigma ^{2}}\sum _{s=1}^{t}\sum _{m=1}^{M}\sum _{A_{s} \in \mathcal{A}_{M,K}}\biggl(\mathbb{P}_{\nu }( A_{s}) \times \left( X_{m,s}^{\top } \theta _{A_{s}( m)} -X_{m,s}^{\prime \top } \theta _{A_{s}( m)}^{\prime }\right)^{2}\biggr)\\
= & \frac{1}{2\sigma ^{2}}\sum _{s=1}^{t}\sum _{m=1}^{M}\sum _{A_{s} \in \mathcal{A}_{M,K}}\biggl(\sum _{k=1}^{K}\mathbb{P}_{\nu}\mathnormal{( A_{s}( m) =k)} \times \mathbb{P}_{\nu }( A_{s} | A_{s}( m) =k)\left( X_{m,s}^{\top } \theta _{k} -X_{m,s}^{\prime \top } \theta _{k}^{\prime }\right)^{2}\biggr)\\
= & \frac{1}{2\sigma ^{2}}\sum _{m=1}^{M}\sum _{k=1}^{K}\sum _{s=1}^{t}\biggl(\left( X_{m,s}^{\top } \theta _{k} -X_{m,s}^{\prime \top } \theta _{k}^{\prime }\right)^{2} \times \sum _{A_{s} \in \mathcal{A}_{M,K}}\mathbb{P}_{\nu }( A_{s} ,A_{s}( m) =k)\biggr)\\
= & \frac{1}{2\sigma ^{2}}\sum _{m=1}^{M}\sum _{k=1}^{K}\sum _{s=1}^{t}\left(\left( X_{m,s}^{\top } \theta _{k} -X_{m,s}^{\prime \top } \theta _{k}^{\prime }\right)^{2}\mathbb{P}_{\nu }( A_{s}( m) =k)\right) .
\end{aligned}
\end{equation*}
\end{proof}

\subsection{Proof of Theorem \ref{thm:lower_bound}}\label{sec:proof_lower}
    
\begin{proof}
We define re-indexing functions $h_{l,L} :[ L]\rightarrow [ L]$ as 
$
h_{l,L}( l') =l'-l+1+L\cdot 1_{\{l'-l+1\leqslant 0\}}
$
for $L\in \{M,K\}$ and $l\in [L]$. Define a base bandit instance
$
\nu ^{( 0)} =\left( \theta _{1}^{( 0)} ,\dotsc ,\theta _{K}^{( 0)} ,\mathcal{X}_{1}^{( 0)} ,\dotsc ,\mathcal{X}_{t}^{( 0)}\right)
$
and a total number of $MK$ other bandit instances
\begin{equation*}
\nu ^{( m,k)} =\left( \theta _{1}^{( m,k)} ,\dotsc ,\theta _{K}^{( m,k)} ,\mathcal{X}_{1}^{( m,k)} ,\dotsc ,\mathcal{X}_{t}^{( m,k)}\right) ,
\end{equation*}
$\forall ( m,k) \in [ M] \times [ K]$, where $\theta _{k}^{( 0)} =X_{m,s}^{( 0)} =0\in \mathbb{R}^{d}$, 
\begin{equation*}
\theta _{k'}^{( m,k)} =\begin{cases}
ae_{h_{k,K}( k')} & \text{, if} \ h_{k,K}( k') \leqslant d\\
\mathbf{0} & \text{, otherwise}
\end{cases} ,
\end{equation*}
\begin{equation*}
X_{m' ,s}^{( m,k)} =\begin{cases}
ae_{h_{m,M}( m')} & \text{, if} \ h_{m,M}( m') \leqslant d\\
\mathbf{0} & \text{, otherwise}
\end{cases} ,
\end{equation*}
for all $( m',k',m,k,s) \in ([ M] \times [ K])^{2} \times [ t]$. Here $e_{d'}$ is the $d$-dimensional basis vector whose $d'$th element is $1$, and $a$ is a real number to be determined later. Therefore, the optimal assignment in $\nu ^{( m,k)}$ is to assign human $m'$ robot $k'$ if $h_{m,M}( m') =h_{k,K}( k') \leqslant d$. The regret of bandit $\nu ^{( m,k)}$ is 
\begin{equation}\label{eq:R_mk}
R_{m,k} =a^{2}\left( dT-\mathbb{E}_{m,k}\left[\sum _{d'=1}^{d} N_{h_{m,M}^{-1}( d') ,h_{k,K}^{-1}( d')}\right]\right) .
\end{equation}

Abbreviate $\mathbb{P}_{\nu ^{( m,k)}}( \cdot )$ as $\mathbb{P}_{m,k}( \cdot )$. Define random variables $N_{m,k} :=\sum _{t=1}^{T} 1\{A_{s}( m) =k\}$ as the number of rounds that human $m$ is assigned robot $k$, for all $m\in [ M]$ and $k\in [ K]$. By Corollary \ref{coro:KL_decomp}, 
\begin{equation*}
\begin{aligned}
  \operatorname{KL}(\mathbb{P}_{0} ,\mathbb{P}_{m,k})
= & \frac{1}{2\sigma ^{2}}\sum _{m'=1}^{M}\sum _{k'=1}^{K}\sum _{s=1}^{t}\biggl((\mathbb{P}_{0}( A_{s}( m') =k') \times \Bigl( X{_{m',s}^{( m,k)}}^{\top } \theta _{k'}^{( m,k)} -X{_{m',s}^{( 0)}}^{\top } \theta _{k'}^{( 0)}\Bigr)^{2}\biggr)\\
= & \frac{a^{4}}{2\sigma ^{2}}\sum _{d'=1}^{d}\mathbb{E}_{0}\left[ N_{h_{m,M}^{-1}( d') ,h_{k,K}^{-1}( d')}\right] .
\end{aligned}
\end{equation*}

Summing both sides over $m$ and $k$, we obtain
\begin{equation}\label{eq:sum_KL}
\begin{aligned}
\sum _{m=1}^{M}\sum _{k=1}^{K}\text{KL}(\mathbb{P}_{0} ,\mathbb{P}_{m,k}) = & \frac{a^{4}}{2\sigma ^{2}}\sum _{m=1}^{M}\sum _{k=1}^{K}\sum _{d'=1}^{d}\mathbb{E}_{0}[ N_{h_{m,M}^{-1}( d') ,h_{k,K}^{-1}( d')}]\\
= & \frac{a^{4}}{2\sigma ^{2}} d\sum _{m=1}^{M}\sum _{k=1}^{K}\mathbb{E}_{0}[ N_{m,k}] =\frac{a^{4}}{2\sigma ^{2}} dTM,
\end{aligned}
\end{equation}
where the second to last equality is true because each $\mathbb{E}_{0}[ N_{m,k}]$ appears exactly $d$ times in the second line and the last equality is true because, for each human $m$, the total number of assignments is $T$ by round $T$.

{\allowdisplaybreaks
Using Lemma \ref{lema:Pinsker},
$
\mathbb{E}_{m,k}[ N_{m',k'}] \leqslant \mathbb{E}_{0}[ N_{m',k'}] +t\sqrt{1-\exp\left( -\text{KL}(\mathbb{P}_{0} ,\mathbb{P}_{m,k})\right)} ,
$
and thus 
\begin{equation}\label{eq:sum_E_mk}
\begin{aligned}
 & \sum _{m=1}^{M}\sum _{k=1}^{K}\sum _{d'=1}^{d}\mathbb{E}_{m.k}[ N_{h_{m,M}^{-1}( d') ,h_{k,K}^{-1}( d')}]\\
\leqslant  & d\sum _{m=1}^{M}\sum _{k=1}^{K}\mathbb{E}_{0}\left[ N_{m,k}\right] + td\sum _{m=1}^{M}\sum _{k=1}^{K}\sqrt{1-\exp\left( -\text{KL}(\mathbb{P}_{0} ,\mathbb{P}_{m,k})\right)}\\
\leqslant  & dTM+ dTMK\sqrt{1-\exp\left( -\frac{ \sum _{m,k}\text{KL}(\mathbb{P}_{0} ,\mathbb{P}_{m,k})}{MK}\right)}\\
= & dTM+dTMK\sqrt{1-\exp\left( -\frac{a^{4} dT}{2K\sigma ^{2}}\right)}
\end{aligned}
\end{equation}
where we derive the term $dTM$ using similarly methods as in  \eqref{eq:sum_KL} and we use Jensen's inequality obtain the second inequality. 
}

By  \eqref{eq:R_mk} and \eqref{eq:sum_E_mk}, we have 
\begin{equation*}
\begin{aligned}
 & \sum _{m=1}^{M}\sum _{k=1}^{K} R_{m,k}\\
= & a^{2}\sum _{m=1}^{M}\sum _{k=1}^{K}\left( dT-\mathbb{E}_{m,k}\left[\sum _{d'=1}^{d} N_{h_{m,M}^{-1}( d') ,h_{k,K}^{-1}( d')}\right]\right)\\
\geqslant  & a^{2}\left( dTMK-dTM-dTMK\sqrt{1-\exp\left( -\frac{a^{4} dT}{2K\sigma ^{2}}\right)}\right)\\
\geqslant  & \sigma \sqrt{dT} M\frac{( K-1)^{2}}{\sqrt{8K}} .
\end{aligned}
\end{equation*}
The proof of the last inequality can be found in the appendix. Therefore, there must be one pair of $( m,k)$ such that
$
R_{m,k} \geqslant \frac{1}{MK} \sigma \sqrt{dT} M\frac{( K-1)^{2}}{\sqrt{8K}} \approx \frac{\sigma }{\sqrt{8}}\sqrt{dTK} .
$
\end{proof}

% \subsection{Proof of Inequality in \texorpdfstring{\eqref{eq:sum_E_mk}}{}}
\subsection{Proof of Inequality in \eqref{eq:sum_E_mk}}

\begin{proof}

\begin{equation*}
\begin{aligned}
 & a^{2}\left( dTMK-dTM-dTMK\sqrt{1-\exp\left( -\frac{a^{4} dT}{2K\sigma ^{2}}\right)}\right)\\
= & dTM\cdot a^{2}\left( K-1-K\sqrt{1-\exp\left( -\frac{a^{4} dT}{2K\sigma ^{2}}\right)}\right)\\
\geqslant  & dTM\cdot a^{2}\left( K-1-K\sqrt{\frac{a^{4} dT}{2K\sigma ^{2}}}\right)\\
= & dTM\cdot K\sqrt{\frac{dT}{2K\sigma ^{2}}} \cdot a^{2}\left(\frac{( K-1)}{K\sqrt{\frac{dT}{2K\sigma ^{2}}}} -a^{2}\right)\\
\geqslant  & dTM\cdot K\sqrt{\frac{dT}{2K\sigma ^{2}}} \cdot \left(\frac{( K-1)}{2K\sqrt{\frac{dT}{2K\sigma ^{2}}}}\right)^{2}\\
= & dTM\cdot \frac{\sigma ( K-1)^{2}}{\sqrt{8KdT}}\\
= & \sqrt{dT} M\cdot \frac{\sigma ( K-1)^{2}}{\sqrt{8K}}
\end{aligned}
\end{equation*}    

\end{proof}

\end{document}